\documentclass{article}
\usepackage{amssymb,amsmath,amsfonts,amsthm,enumerate,color}
\usepackage[toc,page,title,titletoc,header]{appendix}
\usepackage[utf8]{inputenc} % allow utf-8 input
\usepackage[T1]{fontenc}    % use 8-bit T1 fonts
\usepackage[colorlinks=false, pdfborder={0 0 0}]{hyperref}
\usepackage{url}            % simple URL typesetting
\usepackage{booktabs}       % professional-quality tables
\usepackage{nicefrac}       % compact symbols for 1/2, etc.
\usepackage{microtype}      % microtypography
\usepackage{cite}
\usepackage{bm}
\usepackage{amsmath}
\usepackage{amssymb}
\usepackage{amsthm}
\usepackage{algorithm,algorithmic}
\usepackage{subfig}
\usepackage{graphicx}
\usepackage{array}

\usepackage{enumerate}
\newcommand{\mathletter}[1]{
	\expandafter\newcommand\csname b#1\endcsname{\mathbb #1}
	\expandafter\newcommand\csname c#1\endcsname{\mathcal #1}
	\expandafter\newcommand\csname f#1\endcsname{\mathfrak #1}
	\expandafter\newcommand\csname til#1\endcsname{\widetilde #1}
	\expandafter\newcommand\csname ha#1\endcsname{\widehat #1}
	\expandafter\newcommand\csname bf#1\endcsname{\bf #1}
	\expandafter\newcommand\csname s#1\endcsname{\mathsf #1}
}
\def\mathletters#1{\mathlettersB #1,,}
\def\mathlettersB#1,{\ifx,#1,\else\mathletter #1\expandafter\mathlettersB\fi}
\mathletters{A,B,C,D,E,F,G,H,I,J,K,L,M,N,O,P,Q,R,S,T,U,V,W,X,Y,Z}

\newcommand{\mathletterl}[1]{
	\expandafter\providecommand\csname v#1\endcsname{\vec{#1}}
}
\def\mathlettersl#1{\mathlettersC #1,,}
\def\mathlettersC#1,{\ifx,#1,\else\mathletterl #1\expandafter\mathlettersC\fi}
\mathlettersl{a,b,c,d,e,f,g,h,i,j,k,l,m,n,o,p,q,r,s,t,u,v,w,x,y,z}

\def\bt{{\tilde{t}}}

\def\bea{\begin{equation}\begin{aligned} }
\def\ena{\end{aligned}\end{equation} }
\def\bee{\begin{equation}}
\def\ene{\end{equation}}

\renewcommand{\vec}[1]{\mathbf{#1}}

\def\T{\mathsf{T}}

\def\bone{{\mathbf{1}}}
\def\bzero{{\mathbf{0}}}

\def\diag{\text{diag}}

\newtheorem{assum}{Assumption}
\newtheorem{theorem}{Theorem}
\newtheorem{lemma}{Lemma}

\newtheorem{definition}{Definition}
\newtheorem{corollary}{Corollary}
\newtheorem{remark}{Remark}

\makeatletter
\def\@fnsymbol#1{\ensuremath{\ifcase#1\or  \natural \or \dagger\or * \or \ddagger\or
   \mathsection\or \mathparagraph\or \|\or **\or \dagger\dagger
   \or \ddagger\ddagger \else\@ctrerr\fi}}
\makeatother

\title{
Fully Asynchronous Policy Evaluation in Distributed Reinforcement Learning over Networks
}

\author{Xingyu Sha
\thanks{
    X. Sha, J. Zhang and K. You are with the Department of Automation and BNRist, 
    Tsinghua University, Beijing 100084, China. 
    Emails: \texttt{shaxy18@mails.tsinghua.edu.cn, zjq16@mails.tsinghua.edu.cn, youky@tsinghua.edu.cn.}
    }
\and Jiaqi Zhang$^\natural$ \and Keyou You$^\natural$
\and Kaiqing Zhang
\thanks{
    K. Zhang and T. Ba\c{s}ar are with Department of Electrical and Computer Engineering and Coordinated Science Laboratory, University of Illinois at Urbana-Champaign, Urbana, IL 61801 USA.
    Emails: \texttt{kzhang66@illinois.edu, basar1@illinois.edu.}
    } 
\and Tamer~Ba\c{s}ar$^\dagger$ 
}

\begin{document}

\maketitle

\begin{abstract}
	This paper proposes a \emph{fully asynchronous} scheme for the  policy evaluation problem of distributed reinforcement learning (DisRL) over directed peer-to-peer networks. Without waiting for any other node of the network, each node can locally update its value function at any time by using (possibly delayed) information from its neighbors. This is in sharp contrast to the gossip-based scheme where a pair of nodes concurrently update. Though the fully asynchronous setting involves a difficult multi-timescale decision problem, we design a novel stochastic average gradient (SAG) based distributed algorithm and develop a push-pull augmented graph approach to prove its exact convergence at a linear rate of $\mathcal{O}(c^k)$ where $c\in(0,1)$ and $k$ increases by one no matter on which node updates. Finally, numerical experiments validate that our method speeds up linearly with respect to the number of nodes, and is robust to straggler nodes. 
\end{abstract}
    
\section{Introduction}\label{sec1}

Reinforcement learning (RL) aims to guide decision makers (aka agents) to learn optimal policies/strategies by interacting with the environment, which has achieved super-human performance in many tasks \cite{silver2017mastering,mnih2015human}. This work considers the \emph{distributed reinforcement learning (DisRL)} problem \cite{chen2018communication} over a directed {\em peer-to-peer} network, which includes two distinct RL setups depending on the role of the network node. 

One is the so-called multi-agent RL (MARL) in which network nodes are the acting agents of MARL \cite{zhang2018fully, wai2018multi, kar2013cal}, and has proved its capability in solving the problems of e-sports  \cite{vinyals2019grandmaster}, swarm robotics  \cite{kober2013reinforcement}, traffic control \cite{van2016coordinated} and resource allocation \cite{mannion2016dynamic}, etc. The other is the parallel RL where network nodes are designed to jointly solve a large-scale RL problem over decentralized datasets \cite{mnih2016asynchronous,espeholt2018impala}. The celebrated A3C in \cite{mnih2016asynchronous} adopts a master-worker configuration for parallel RL where a central node communicates with all the other nodes and is a special case of DisRL over networks.

In DisRL over networks, each node has access to a local dataset and only communicates with a subset of agents to learn the policy locally, either synchronously or asynchronously.  In a synchronous model, all nodes communicate and update their learning variables in synchronized rounds via a global clock/counter \cite{zhang2018fully,qu2019value,wai2018multi}, which  may be hard to implement in large-scale systems \cite{bertsekas1989parallel}, and suffer from deadlocks \cite{assran2017empirical,lian2017asynchronous}. While  synchronous algorithms are easier to design and analysis, their efficiency is greatly dragged by the slow nodes. 

To facilitate the implementation and improve the computational efficiency of the DisRL, we adopt the {\em fully asynchronous} model from \cite{herz1993distributed} where every node can update at {\em any} time by using possibly delayed information from neighbors. Since we aim for {\em directed} communications, each node does not need to wait for others and is free of deadlocks. This is in sharp contrast to the random gossip-based asynchronous methods where only a pair of nodes concurrently update via pairwise averaging per round \cite{lian2017asynchronous, xu2017convergence,assran2021asynchronous}. Clearly, such a coordination between paired nodes may create deadlocks in practice and is vulnerable to information delays. It can also be problematic if a node is unable to respond or has only access to its local dataset \cite{lian2017asynchronous}. In fact, the advantages of the fully asynchronous model has been well documented both empirically and theoretically for distributed optimization problems, see e.g.  \cite{assran2017empirical,zhang2019asyspa,tian2020achieving}. In this work, we study the fully asynchronous policy evaluation of the DisRL over directed  networks.

The policy evaluation of MARL over networks is a rising topic and has been studied in \cite{doan2019finite, cassano2020multi,wai2018multi, ren2019communication, ding2019fast}. As \cite{wai2018multi, cassano2020multi}, we adopt the primal-dual reformulation  with a mean square projected Bellman error (MSPBE) objective function for policy evaluation.  However, the striking difference is that they only consider the synchronous case with {\em undirected} networks. In fact,  the information flow over a fully asynchronous network has only one direction since each node cannot expect any response from other nodes even if the physical communication network  is bidirected. Thus the directed network is an inevitable issue in our fully asynchronous DisRL. Note that the extension from the undirected to directed networks is nontrivial in the distributed setting \cite{xie2018distributed}.

In the policy evaluation, we usually have to use the stochastic gradient (SG). While the SG descend method bears a vast body of literature in the centralized setting \cite{bottou2018optimization}, it is not the case in the distributed setting, and how to design effective distributed learning algorithms with SGs has attracted an increasing attention \cite{xin2020general}. For example, the distributed methods with SGs in \cite{xin2019variance,yuan2019variance,mokhtari2016dsa,hendrikx2019asynchronous} are all restricted to undirected networks. Though it has been resolved in 
\cite{qureshi2020push} via the SAGA method, it involves a division operator per iteration with a denominator that could be arbitrarily close to zero, leading to the numerical instability. It is also unclear how to extend the above mentioned works to the fully asynchronous model.

In this work, we integrate the  push-pull strategy in  \cite{xin2018linear,pu2021push}, which however only solve the gradient-based distributed optimization problems in a synchronous way, with the celebrated stochastic average gradient (SAG) method, and design a novel fully asynchronous push-pull SAG (APP-SAG) to find a saddle point of the primal-dual reformulation of the DisRL over directed  networks.

Since the fully asynchronous setting involves a  multi-timescale decision problem, the convergence proof of APP-SAG is challenging. To this end, we develop a push/pull augmented graph approach to prove the exact linear convergence of APP-SAG. Though the graph augmentation is  proposed by \cite{nedic2010convergence}, they only use one type of virtual nodes for each node, called the {\em pull} augmented digraph in this work. To describe APP-SAG with a single timescale, we further design a {\em push} augmented digraph. The push-pull augmentation scheme is also different from the augmented graph in \cite{tian2020achieving} where multiple virtual edges are introduced for each edge of the network. From the worst-case point of view, we  show that APP-SAG converges linearly at a deterministic rate of $\mathcal{O}(c^k)$ where $c\in(0,1)$ mainly depends on the network topology, the level of asynchronism and the maximum delay lengths, and $k$ increases by one no matter on which node updates.

In fact, the A3C in \cite{mnih2016asynchronous}  is the first asynchronous parallel RL, where the asynchronicity is shown to have stabilizing advantages in the training process. In GALA \cite{assran2019gossip},  the central node of A3C is replaced with a group of learners over a directed  network, which has been shown to outperform A3C in both data efficiency and robustness.
Unfortunately, it cannot ensure convergence guarantees for policy evaluation. To the best of our knowledge, how to design a fully asynchronous distributed learning algorithm for the policy evaluation of DisRL over directed  networks remains an open problem.  Even for gradient-based distributed optimization, \cite{assran2021asynchronous} show that a naive extension of synchronous gradient-push algorithm to the fully asynchronous model may diverge. Though it has been resolved in \cite{zhang2019asyspa} via adaptive tuning stepsizes or gradient tracking technique in \cite{tian2020achieving, zhang2019asynchronous}, none of them can ensure an {\em exact} convergence with SGs.

Finally,  APP-SAG is compared with PD-DistIAG \cite{wai2018multi} and FDPE \cite{cassano2020multi} on a real networked MARL problem in \cite{qu2020scalable}. We also illustrate its speedup property on the mountaincar problem from \cite{sutton1998introduction}.  To summarize, the main contributions of this work at least include:
\begin{enumerate}[(a)]
\item We have designed a fully asynchronous APP-SAG with an exact linear convergence for the policy evaluation of DisRL. 
\item By inspecting the features of APP-SAG, we introduce the concepts of push/pull augmented digraphs to clarify our augmented approach.  
\item The celebrated SAG has been extended to the distributed setting over directed networks. 
\end{enumerate}

The rest of the paper is organized as follows. In Section \ref{bigsec2}, we introduce the formulation of DisRL. In Section \ref{bigsec3}, we propose APP-SAG with linear convergence guarantees. In Section \ref{aug}, we construct the push/pull augmented digraphs to reformulate APP-SAG. In Section \ref{bigsec4}, we use the numerical results to validate the performance of APP-SAG. Some concluding remarks are included in Section \ref{bigsec5}. A formal proof is postponed to Appendix \ref{prooj_theo}.

Main notations used in this paper are given below. (1)
	$[\vec x_1,\ldots,\vec x_n]$ and $[\vec x_1;\ldots;\vec x_n]$ denote the horizontal and vertical stack of vectors
	$\vec x_1,\ldots,\vec x_n$, respectively.
	(2) $\Vert\cdot\Vert_2$ denotes the $l_2$-norm of a vector or matrix. $\Vert\cdot\Vert_\sF$ denotes the matrix Frobenius norm. $\Vert \cdot\Vert_{\bm{D}}$ denotes the
	${\bm{D}}$-norm, i.e., $\Vert \vx\Vert_{\bm{D}}=({\vx^\T\bm{D}}\vx)^{1/2}$ if the matrix ${\bm{D}}$ is semi-positive definite.
	(3) $[\bm{A}]_{ij}$ denotes the $(i,j)$-th element of $\bm{A}$. $\lambda_{\text{max}}(\bm{A})$ and $\lambda_{\text{min}}(\bm{A})$ denote the largest and smallest eigenvalues of $\bm{A}$, respectively.
	(4) $\bone_n$ and $\bzero_n$ denote the $n$-dimensional vector with all ones and all zeros, respectively.
	(5) $\bm{A}$ is nonnegative. Then it is a row-stochastic matrix if $\bm{A}\bone=\bone$, and is column-stochastic if $\bone^\T\bm{A}=\bone^\T$.

\section{Problem formulation}\label{bigsec2}
\subsection{DisRL over networks}\label{sec2-1}

Let $\mathcal G=(\mathcal V,\mathcal E)$ be a  directed  network, where $\mathcal V=\{1,\ldots,n\}$ is the set of network nodes and $\mathcal E\subseteq \mathcal V\times \mathcal V$ is the set of edges. 
A directed edge $(i,j)\in \mathcal E$ means that node $i$ can directly send information to node $j$.
Denote $\mathcal N_{in}^i=\{j|(j,i)\in \mathcal E\}\cup \{i\}$ and $\mathcal N_{out}^i=\{j|(i,j)\in \mathcal E\}\cup \{i\}$ as the sets of in-neighbors and out-neighbors of node $i$, respectively. 

DisRL over a  network is based on a finite discrete Markov decision process (MDP) with a tuple $(\mathcal{S}, \prod_{i=1}^n\mathcal{A}_i, \mathcal{P}, \{\mathcal{R}_i\}_{i=1}^n, \gamma, \mathcal G)$, 
where  $\mathcal{S}$ and $\mathcal{A}_i$ denote the state space and the action space of node $i$ in $\mathcal G$, respectively. $\mathcal{P}(\vec s'|\vec s,\vec a)$ is the probability of transition from $\vec s\in\mathcal{S}$ to $\vec s'\in\mathcal{S}$ under an action $\vec{a}\in\mathcal{A}:=\prod_{i=1}^n\mathcal{A}_i$. 
$\mathcal{R}_i=\mathcal{R}_i(\vec s,\vec a, \vec s')$ is the reward perceived by a single node $i$. $\gamma\in(0, 1)$ is the discount factor. $\mathcal G$ is used to model the interactions among nodes. 

A stochastic policy $\bm \pi(\vec{a}|\vec{s})$ is the conditional probability of taking action $\vec{a}$ given a specific state $\vec{s}$. 
If we focus on a target policy $\bm \pi$, the fixed transition probability matrix is denoted as $\bm{\mathcal{P}}^{\bm{\pi}}$, whose $(\vec s,\vec s')$-th entry is given by $[\bm{\mathcal{P}}^{\bm{\pi}}]_{s,s'}=\mathcal{P}(\vec s'|\vec s,\bm\pi)=\sum_{\vec{a}\in\mathcal{A}}{\bm{\pi}}(\vec{a}|\vec{s})\cdot \mathcal{P}(\vec s'|\vec s,\vec a)$. 
Let $\mathcal{R}^{\bm{\pi}}(\vec{s})=\frac1n \sum_{i=1}^n\mathcal{R}_i^{\bm{\pi}}(\vec{s})$, where $\mathcal{R}_i^{\bm{\pi}}(\vec{s}) = \mathbb{E}_{\mathcal{P}, \vec{a}\sim {\bm{\pi}}(\cdot|\vec{s})}[\mathcal{R}_i(\vec{s},\vec{a},\vec{s}')]$ is the expected reward of node $i$ at state $\vec{s}$ if the group follows $\bm{\pi}$.

\subsection{Policy evaluation over networks}\label{sec2-2}
In this work, we solve the mean squared projected Bellman error (MSPBE) problem in both MARL and parallel RL setups with experiences over networks.

Under a given policy ${\bm{\pi}}$, value function $V^{\bm{\pi}}(\vec{s})$ is defined as $
  V^{\bm{\pi}}(\vec{s})=
  \mathbb{E}\left[\sum_{t=0}^\infty \gamma^t \mathcal{R}^{\bm{\pi}}(\vec{s}(t))|\vec{s}(0)=\vec{s},\bm{\pi},\mathcal{P} \right]
$, and we simply write its vector form as $\bm{V}^{\bm\pi}\in \mathbb{R}^{|\mathcal{S}|}$, which satisfies the Bellman equation \cite{sutton1998introduction}
$
  \bm{V}^{\bm{\pi}}=\bm{\mathcal{R}}^{\bm{\pi}} + \gamma\bm{\mathcal{P}}^{\bm{\pi}}\bm{V}^{\bm{\pi}},
$
where $\bm{\mathcal{R}}^{\bm\pi}$ is obtained by stacking up ${\mathcal{R}}^{\bm\pi}(\vec s)$. In this work, we adopt a linear approximator to evaluate policy, e.g.
$V^{\bm \pi}(\vec s)\approx \bm{\phi}^\T(\vec s)\bm{\theta},$
where $\bm{\phi}(\vec s)\in\mathbb{R}^d$ is a feature vector
corresponding to $\vec s\in\cS$. 
That is finding a vector $\bm{\theta}\in \mathbb{R}^d$ such that $\bm{V}_{\bm{\theta}}=\bm{\Phi\theta}\approx \bm{V}^{\bm{\pi}}$ by minimizing MSPBE.
$
  J(\bm{\theta}) = \frac12 \Vert \bm{\Pi}_{\bm{\Phi}}( \bm{V}_{\bm{\theta}} - \gamma \bm{P}^{\bm{\pi}}\bm{V}_{\bm{\theta}} - \bm{\mathcal{R}}^{\bm{\pi}} ) \Vert^2_{\bm{D}} + \frac{\rho}{2} \Vert \bm{\theta}\Vert^2,
$
where $\bm{D}=\mbox{diag}(\bm{\mu}^{\bm{\pi}}(\vec s))$
 and $\bm{\mu}^{\bm{\pi}}(\vec s)$ is the stationary distribution of states under ${\bm{\pi}}$, $\bm{\Pi}_{\bm{\Phi}}$ is the projection onto the linear subspace $\{\bm{\Phi\theta}\}$, and 
$\frac{\rho}{2} \Vert \bm{\theta}\Vert^2$ is the regularization term. 
By \cite{sutton2009MSPBE}, the MSPBE loss function is equivalent to
\begin{equation}\label{equ:2-2:MSPBE2}
  J(\bm{\theta})=\frac12 \Vert\bm{A}\bm{\theta}-\bm{b}\Vert^2_{\bm{C}^{-1}} + \frac{\rho}{2} \Vert \bm{\theta}\Vert^2
\end{equation}
where $
  \bm{A}=\mathbb{E}_{\mu^{\bm{\pi}}}[\bm{\phi}_t(\bm{\phi}_t-\gamma \bm{\phi}_{t+1})^\T],
  \bm{b}=\mathbb{E}_{\mu^{\bm{\pi}}}[\mathcal{R}^{\bm{\pi}}(\vec{s}(t))],
  \bm{C}=\mathbb{E}_{\mu^{\bm{\pi}}}[\bm{\phi}_t\bm{\phi}_t^\T]$,
and $\bm{\phi}_t$ is the short for $\bm{\phi}(\vec s(t))$.

\subsubsection{Parallel RL}

Parallel RL adopts multiple nodes to learn
an optimal policy $\bm{\pi}$. In this case, the MDP tuple of each node $i$ is $(\mathcal{S}_i, \mathcal{A}_i, \mathcal{P}_i, \mathcal{R}_i, \gamma)$,
where $\mathcal{S}_i=\mathcal{S}, \mathcal{A}_i=\mathcal{A}, \mathcal{P}_i=\mathcal{P}$ since the MDP model is identical. 
The local rewards may differ even under the same distribution, e.g. $\mathbb{E}_{\cP,\mu^{\bm{\pi}}}[\mathcal{R}_i(\vec s, \vec a, \vec s')]=\mathbb{E}_{\cP,\mu^{\bm{\pi}}}[\mathcal{R}_j(\vec s, \vec a, \vec s')], \forall (\vec s, \vec a, \vec s')\in \mathcal{S}\times \mathcal{A}\times\mathcal{S}, \forall i,j\in \mathcal{V}.$ 

Each node collects a finite sequence of samples $\cX_i=\{ \vec{s}_{i,p},\vec{a}_{i,p}, \mathcal{R}_{i}(\vec{s}_{i,p},\vec{a}_{i,p}, \vec{s}_{i,p+1}) \}_{p=1}^{m_i}$ in the learning process.
Let ${m}=\sum_{i=1}^{n}m_i$, and $\bm{A}$ in \eqref{equ:2-2:MSPBE2} is estimated as 
$\widehat{\bm{A}}=\frac{m_i}{{m}}\sum_{i=1}^n\widehat{\bm{A}}_i,$ 
where node $i$ locally computes 
$\widehat{\bm{A}}_i = \frac{1}{m_i} \sum_{p=1}^{m_i}\widehat{\bm{A}}_{i,p}$ and $\widehat{\bm{A}}_{i,p}=\bm{\phi}_{i,p}(\bm{\phi}_{i,p}-\gamma \bm{\phi}_{i,p+1})^\T$. 
Note that ${\bm{b}}$ and ${\bm{C}}$ are estimated analogously.
\par Under the distributed setting for \eqref{equ:2-2:MSPBE2}, each node maintains a local copy $\{\bm{\theta}_i\}$, leading to a consensus-based form
\begin{equation}\label{equ:2-2:primal-problem}\left.\begin{aligned}
  &\min_{\bm{\theta}} \frac12 \left\Vert \frac{m_i}{{m}} \sum_{i=1}^n\left(\widehat{\bm{A}}_i\bm{\theta}_i-\widehat{\bm{b}}_i \right)\right\Vert^2_{\widehat{\bm{C}}^{-1}} + \frac{m_i}{{m}}\sum_{i=1}^n\frac \rho2 \Vert \bm{\theta}_i\Vert^2\\
  & \text{subject to}\ \bm{\theta}_1=\bm{\theta}_2=\ldots=\bm{\theta}_n.
\end{aligned}\right.\end{equation}

\subsubsection{Collaborative MARL}
In MARL, the nodes are heterogeneous agents over networks \cite{zhang2019multi}. Each agent observes a global state $\vec s$ and chooses an action $\vec a_i\in \mathcal{A}_i$ under its local policy $\bm{\pi}_i(\vec a_i|\vec s)$.  The joint action $\vec a=(\vec{a}_1,\ldots,\vec{a}_n)$ is then executed by the group.  Thus the state transition 
depends on the joint policy $\bm{\pi}=(\bm{\pi}_1,\ldots,\bm{\pi}_n)$. 

Under $\bm \pi$, the distributed agents together collect a joint trajectory $\{ \vec{s}_{p},\vec{a}_{p} \}_{p=1}^{m_i}$, where $m_i=m_j,\forall i,j\in\cV$, but with different reward sequences $\{\mathcal{R}_i(\vec{s}_{p},\vec{a}_{p}, \vec{s}_{p+1}) \}_{p=1}^{m_i}$.
The agents aim to collaboratively maximize the group mean reward $\mathcal{R}(\vec{s}_{p},\vec{a}_{p},\vec{s}_{p+1})=\frac1n\sum_{i=1}^n\mathcal{R}_i(\vec{s}_{p},\vec{a}_{p}, \vec{s}_{p+1})$.
The local data of $i$ is $\cX_i=\{ \vec{s}_{p},\vec{a}_{p}, \mathcal{R}_i(\vec{s}_{p},\vec{a}_{p}, \vec{s}_{p+1}) \}_{p=1}^{m_i}$.
Then, the empirical MSPBE of MARL is identical to \eqref{equ:2-2:primal-problem}, except the minor difference that $m_i=m/n,\widehat{\bm{C}}_i=\widehat{\bm{C}}$ and $\widehat{\bm{A}}_i=\widehat{\bm{A}}$ for all $i\in\mathcal{V}$.
In this case, $\cX_i$ are different though the sample sizes are identical.

\subsection{MSPBE saddle-point reformulation}
As in \cite{wai2018multi, cassano2020multi}, we adopt the conjugate form of the original $\bm{C}^{-1}$-norm, i.e.,
$$
\frac12 \Vert\bm{A}\bm{\theta}-\bm{b}\Vert^2_{\bm{C}^{-1}} = \max_{\bm{\omega}} \big( \bm{\omega}^\T(\bm{A}\bm{\theta}-\bm{b}) - \frac12\bm{\omega}^\T\bm{C}\bm{\omega} \big),
$$
and rewrite \eqref{equ:2-2:primal-problem} as
\bee\begin{aligned}\label{equ:2-3:primal-dual-problem}
&\min_{\bm{\theta}_i} \max_{\bm{\omega}_i} J(\bm{\theta}_i, \bm{\omega}_i)=\frac{1}{{m}} \sum_{i=1}^n\sum_{p=1}^{m_i} J_{i,p}(\bm{\theta}_i, \bm{\omega}_i)\\
&\text{subject to}\ \ \bm{\theta}_1=\ldots=\bm{\theta}_n, \bm{\omega}_1=\ldots=\bm{\omega}_n
\end{aligned}\ene
where $J_{i,p}(\bm{\theta}_i, \bm{\omega}_i) = \bm{\omega}_i^\T(\widehat{\bm{A}}_{i,p}\bm{\theta}_i-\widehat{\bm{b}}_{i,p}) - \frac12\bm{\omega}_i^\T\widehat{\bm{C}}_{i,p}\bm{\omega}_i + \frac \rho2 \Vert \bm{\theta}_i\Vert^2.$

Clearly, 
$
\nabla_{\bm{\theta}_i} J_{i,p}(\bm{\theta}_i,\bm{\omega}_i) = \widehat{\bm{A}}_{i,p}^\T\bm{\omega}_i + \rho \bm{\theta}_i,
\nabla_{\bm{\omega}_i} J_{i,p}(\bm{\theta}_i,\bm{\omega}_i) = \widehat{\bm{A}}_{i,p}\bm{\theta}_i- \widehat{\bm{C}}_{i,p}\bm{\omega}_i-\widehat{\bm{b}}_{i,p}.
$ 
Throughout the paper, let $\vec z_i=[\bm{\theta}_i; \bm{\omega}_i]$, $\cM_i=\{1,\ldots,m_i\}$, and 
$\nabla J_{i, p_i}(\vec z_i) = [\nabla_{\bm{\theta}_i} J_{i,p_i}(\bm{\theta}_i,\bm{\omega}_i);   
-\nabla_{\bm{\omega}_i} J_{i,p_i}(\bm{\theta}_i,\bm{\omega}_i)]$.

\section{The APP-SAG and linear convergence}\label{bigsec3}
\subsection{The APP-SAG}\label{description}

\renewcommand{\algorithmicrequire}{\textbf{Initialize}} 
\renewcommand{\algorithmicrepeat}{\textbf{Repeat}} 
\renewcommand{\algorithmicuntil}{\textbf{Until}} 
\begin{algorithm}[t!]
	\caption{The APP-SAG  -- from the view point of node $i$}
	\label{alg:proposed}
	\begin{algorithmic}[1]
	\REQUIRE Set arbitrary $\vec z_i$, and let $\hat{\vg}_{i,p}\leftarrow\nabla J_{i,p}(\vec z_i)$ for all $p\in\cM_i$. 
	\STATE Compute a scaled SAG  
	 $\vec y_i\leftarrow\frac1{{m}}\sum_{p=1}^{m_i}\nabla J_{i,p_i}(\vec z_i)$ and create local buffers $\mathcal{Z}_i,\mathcal{Y}_i$.  
	\STATE Let $\tilde{\vec z}_i\leftarrow\vec z_i$, $\tilde{\vec y}_i\leftarrow\vec y_i/|\mathcal{N}_{out}^i|$, and broadcast $\tilde{\vec z}_i$ and $\tilde{\vec y}_i$ to all out-neighbors.
	\REPEAT
	  \STATE  Keep receiving $\tilde{\vec z}_j$, $\tilde{\vec y}_j$ from in-neighbors and save to $\mathcal{Z}_i$ and $\mathcal{Y}_i$, respectively, until being activated for a new update.
	  \STATE  Query the local buffers to compute
		\bea
		\label{gg4}
		\vec z_i & \leftarrow\mbox{avg}(\mathcal{Z}_i), \vec y_i \leftarrow \mbox{sum}(\mathcal{Y}_i),
		\ena
		where $\text{avg}(\cdot)$ and $\text{sum}(\cdot)$ return the average and sum over their arguments respectively.
	  \STATE
	 Randomly pick a sample indexed by $p_i\in \cM_i$ from the local sample set $\cX_i$ and update as 
	  \begin{align}
		\label{gg5}
	  \vec{g}_{i} &\leftarrow \nabla J_{i,p_i}(\vec z_i), \\
	  \label{gg6}
	  \vec y_i & \leftarrow \vec y_i+\frac1{{m}}\vec{g}_{i} -\frac1{{m}}\hat{\vg}_{i,p_i},\\
	  \label{gg7}
	  \hat{\vg}_{i,p_i}&\leftarrow \vg_i,
	  \end{align}
	  \STATE Update the messages to send
	  \begin{equation}\begin{aligned}\label{algo_update}
			\tilde{\vec z}_i \leftarrow
			\vec z_i -\begin{bmatrix}\eta_1\bm{I}_d &\bm{0} \\\bm{0}& \eta_2 \bm{I}_d\end{bmatrix}\vec y_{i},\ 
			\tilde{\vec y}_i\leftarrow\vec y_i/|\mathcal{N}_{out}^i|,
	  \end{aligned} 
	  \end{equation} 
	  where $\eta_1,\eta_2 >0$ are two positive stepsizes.
	  \STATE Broadcast $\tilde{\vec z}_i$ and $\tilde{\vec y}_i$
	   to all out-neighbors, and empty $\mathcal{Z}_i$ and $\mathcal{Y}_i.$
	\UNTIL {a stopping criteria is satisfied, e.g. $\Vert \vec y_i\Vert<\epsilon$ for a predefined $\epsilon>0$.}
	\end{algorithmic}
\end{algorithm}

In this subsection, we propose a fully asynchronous push-pull stochastic average gradient (APP-SAG) algorithm to solve \eqref{equ:2-3:primal-dual-problem}. See Algorithm \ref{alg:proposed} for details. Note that the push-pull strategy is introduced in  \cite{pu2021push, xin2018linear} to solve the gradient-based distributed optimization problems in a synchronous way.

We first explain how to implement Algorithm \ref{alg:proposed} in a fully asynchronous way over a directed  network.  Every node keeps receiving $\tilde{\vec z}_j$ and $\tilde{\vec y}_j$ from its in-neighbors and respectively copying them to its local buffers $\mathcal{Z}_i$ and $\mathcal{Y}_i$, both of which may contain multiple receptions from a single in-neighbor. 
If node $i$ starts to update at {\em any} time, it randomly picks a sample\footnote{It also allows to use a random mini-batch where the SAG can be obtained via the mini-batch size.} from its local sample set $\cX_i$  and queries its buffers to perform \eqref{gg4}-\eqref{algo_update}. 
Then, it broadcasts the updated vectors $\tilde{\vec z}_i$ and $\tilde{\vec y}_i$ to all out-neighbors, which may also be subject to unpredictable but bounded delays, and empties both buffers. Clearly, there is no synchronization between nodes and is fully asynchronous, which is in sharp contrast with the gossip-based model \cite{lian2017asynchronous, xu2017convergence,assran2021asynchronous} as the latter involves a pair of neighbors to concurrently update. 

Now, we explain the key idea of APP-SAG.  Inspired by our previous work \cite{zhang2019asynchronous}, we use local buffers to handle asynchronicity issues and the \emph{push-pull step} in \eqref{gg4} for directed networks, respectively. 
However, \cite{zhang2019asynchronous} does not involve any dual problem in \eqref{equ:2-3:primal-dual-problem} and the SG issues. As also noticed by \cite{xin2020general}, the use of SGs is non-trivial in the distributed setting and is more involved than that in the centralized setting. Particularly, a direct use of SG in distributed gradient methods cannot ensure exact convergence \cite{xin2020general}.

Thus, we extend the centralized SAG \cite{schmidt2017minimizing} to the fully asynchronous distributed setting over directed networks. Specifically, consider using the SAG method to solve a finite sample average problem, i.e.,
$$
\min_{\vx\in\bR^d} f(\vx)=\frac{1}{m} \sum_{p=1}^m f_p(\vx).
$$
Then, the SAG at the $k$-th iteration is $\vh^k=\frac{1}{m}\sum_{p=1}^m\hat{\vg}^k_p$ where $\hat{\vg}^k_p$ is  updated as
\begin{align*}
	\left\{
	\begin{array}{ll}
		\hat{\vg}^k_{p}\leftarrow \nabla f_{p}(\vx^k), & \text{if}\ p=p^k,\\
		\hat{\vg}^k_{p}\leftarrow \hat{\vg}^{k-1}_{p}, & \text{if}\ p\neq p^k,
	\end{array}
	\right.
\end{align*}
and $p^k$ is randomly chosen from sample set. That is, the SAG is iteratively updated as
\begin{equation}\label{sagin}
\vh^k\leftarrow\vh^{k-1}+\frac{1}{m}\left(\nabla f_{p^k}(\vx^k)-\hat{\vg}^{k-1}_{p^k}\right).
\end{equation}
In comparison, $\vec y_i$ in Algorithm \ref{alg:proposed} aims to distributedly track the scaled SAG of  \eqref{equ:2-3:primal-dual-problem} in a fully asynchronous way. Jointly with \eqref{algo_update}, our distributed algorithm then resembles the centralized SAG, which enjoys the exact linear convergence. Though more technical, one can follow a similar idea in \cite{zhang2019asynchronous} to show that $\vy_i$ is able to asymptotically reach consensus on a scaled value of $\bar{\vy}=\sum_{i=1}^n \vy_i$. Thus, we only need to check that $\bar{\vy}$ is in fact the SAG of  \eqref{equ:2-3:primal-dual-problem}. This can be done via mathematical induction arguments. Notice from the initialization of Algorithm \ref{alg:proposed}, that $\bar{\vy}\leftarrow\frac1{{m}}\sum_{i=1}^n\sum_{p=1}^{m_i} \nabla J_{i,p}(\vec z_i)$. Now, suppose that $\bar{\vy}$ is already an SAG. If some node $i$ has completed a new update in Algorithm \ref{alg:proposed} and there is no transmission delay between nodes, it follows from \eqref{gg4}-\eqref{gg7} and the definition of $\tilde{\vy}_i$ that 
\bee\label{sumy}
\bar{\vy}\leftarrow \bar{\vy}-\vy_i + \sum\nolimits_{j\in \mathcal{N}_{out}^i}\tilde{\vy}_i = \bar{\vy}+\frac{1}{{m}}(J_{i,p_i}(\vec z_i)-\hat{\vg}_{i,p_i})
\ene
where $\hat{\vg}_{i,p_i}$ is given in \eqref{gg6} and is the latest SG before this update. In view of \eqref{sagin}, it is clear that $\bar{\vy}$ is indeed the SAG of \eqref{equ:2-3:primal-dual-problem}.

Note that the local buffers are only used to illustrate the key ideas, and can be removed in practice as \eqref{gg4} can be recursively computed. The use of SAG requires each node to store $m_i$ SGs, to compute only one SGs per iteration.

\subsection{Linear convergence}

Under the following assumptions, our main theoretical result rigorously quantifies the linear convergence of Algorithm \ref{alg:proposed}.
\begin{assum}\label{assum1}
	\begin{enumerate}[(a)]
		\item The digraph $\cG$ is strongly connected, i.e., any pair of nodes can be connected via a sequence of consecutive directed edges.
		\item 
		Let $\mathcal T=\{ t^k \}_{k\ge 1}$ be an increasing sequence to record the updating time instants of all nodes, e.g., 
		$t^k \in \mathcal T$ if and only if one node\footnote{If more than one node start their updates exactly at the same moment, we can simply regard them as updating successively, which will not violate our analysis in section \ref{aug}.}
 starts a new update at time instant $t^k$.	There exists a constant $b>0$ such that every node $i$  at least completes a new update, which is also received by each of its out-neighbors, in the time interval $[t^{k},t^{k+b})$ for all $t^k \in \mathcal T$.
	\end{enumerate}
\end{assum}

\begin{assum}\label{assum2}
	\begin{enumerate}[(a)]
		\item The sample size is sufficiently large such that $\widehat{\bm{A}}$ is full-rank and $\widehat{\bm{C}}$ is positive-definite.
		\item Each sample in the local sample set $\cX_i$ is selected at least once in every $K$ iterations.
	\end{enumerate}
\end{assum}
Assumption \ref{assum1}(a) is necessary, since otherwise there exist some nodes unable to be accessed by other nodes. Assumption \ref{assum1}(b) is easy to satisfy in practice if the communication delays are bounded and each node only takes finite time to complete one update. If the time interval between two consecutive updates is infinite for some node, it implies that such a node cannot participate in the learning process. 
It should be noted that the implementation of Algorithm \ref{alg:proposed}  does not depend on the parameter $b$.  Assumption \ref{assum2}(a) ensures the existence of a unique optimal solution for \eqref{equ:2-3:primal-dual-problem}.
Assumption \ref{assum2}(b) is also simple and can be easily satisfied with a random reshuffle rule \cite{haochen2019random}.

Let the optimal solution of \eqref{equ:2-3:primal-dual-problem} be $\vz^*$ and denote the latest value of $\vz_i$ before $t^k$ in
node $i$ by $\vz_i^k$.  Our main theoretical result is given below.

\begin{theorem}{\label{main_theo}}
	Let $\eta=\eta_1, \zeta=\eta_2/\eta_1$.
	Suppose that Assumptions \ref{assum1}-\ref{assum2} hold and
	$$\eta \in \left(0, \frac{\alpha \kappa^4(1-\kappa)^2}{72\beta^3 n^3 b^6 K^3\tilde{t}^2}\right),\ \zeta>\frac{4\rho+4\lambda_{\text{max}}(\widehat{\bm{A}}^\T\widehat{\bm{C}}^{-1}\widehat{\bm{A}})}{\lambda_{\text{min}}(\widehat{\bm{C}}^{-1})}.$$
	Then,
	$\|\vz_i^k-\vz^*\| $ 
	converges to zero at a linear rate $\mathcal{O}(c^k)$,
	where 
	$c=\max\left\{\sqrt[\tilde{t}+1]{\frac12 + \kappa^{-1}\mu n}, \sqrt[b+1]{1 - \eta \alpha\kappa n/2} \right\}$,
	and $k$ is increased by one no matter on which node has completed an update. The positive constants $\alpha, \beta, \kappa, \tilde{t}, \mu$ are given in Appendix A.1, and $\kappa^{-1}\mu n<1/2$ by Corollary \ref{coro1}.
	\end{theorem}

\begin{remark}\label{rem-zeta}
	It should be noted that
	$\eta_2=\eta\zeta\in
	\left(0, \frac{2m\beta}{\psi}\eta\right),$ where $\psi=\lambda_{\text{max}}(\widehat{\bm{C}})$. See Lemma \ref{lemma0-1} in Appendix \ref{preliminaries} for details.
\end{remark}

We emphasize that the deterministic convergence rate in Therorem \ref{main_theo} is established from the worst-case point of view, which is quite different from the random gossip scheme \cite{lian2017asynchronous, xu2017convergence,assran2021asynchronous}.  More importantly, $k$ is unknown to any node. In synchronous algorithms, $k$ is known to each node and is increased only if all nodes have completed their new updates, which largely confirms the computational efficiency of APP-SAG.

The proof of Theorem \ref{main_theo} is quite lengthy and its details are postponed to Appendix A. In Section \ref{aug}, we develop an augmented system approach to provide some key results to the proof.
\section{The push-pull augmented graph approach}\label{aug}

\begin{figure}[!t]
	\centering
	\captionsetup{justification=raggedright}
	\subfloat[]{\label{fig1-sec1}\includegraphics[width=0.3\linewidth]{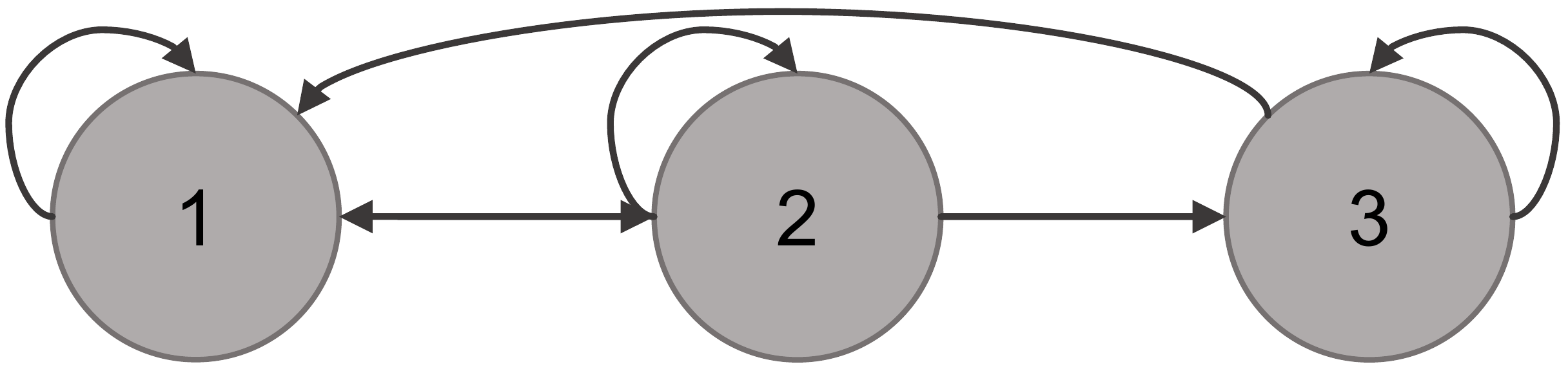}}\\
		\subfloat[]{\label{fig1-sec2}\includegraphics[width=0.28\linewidth]{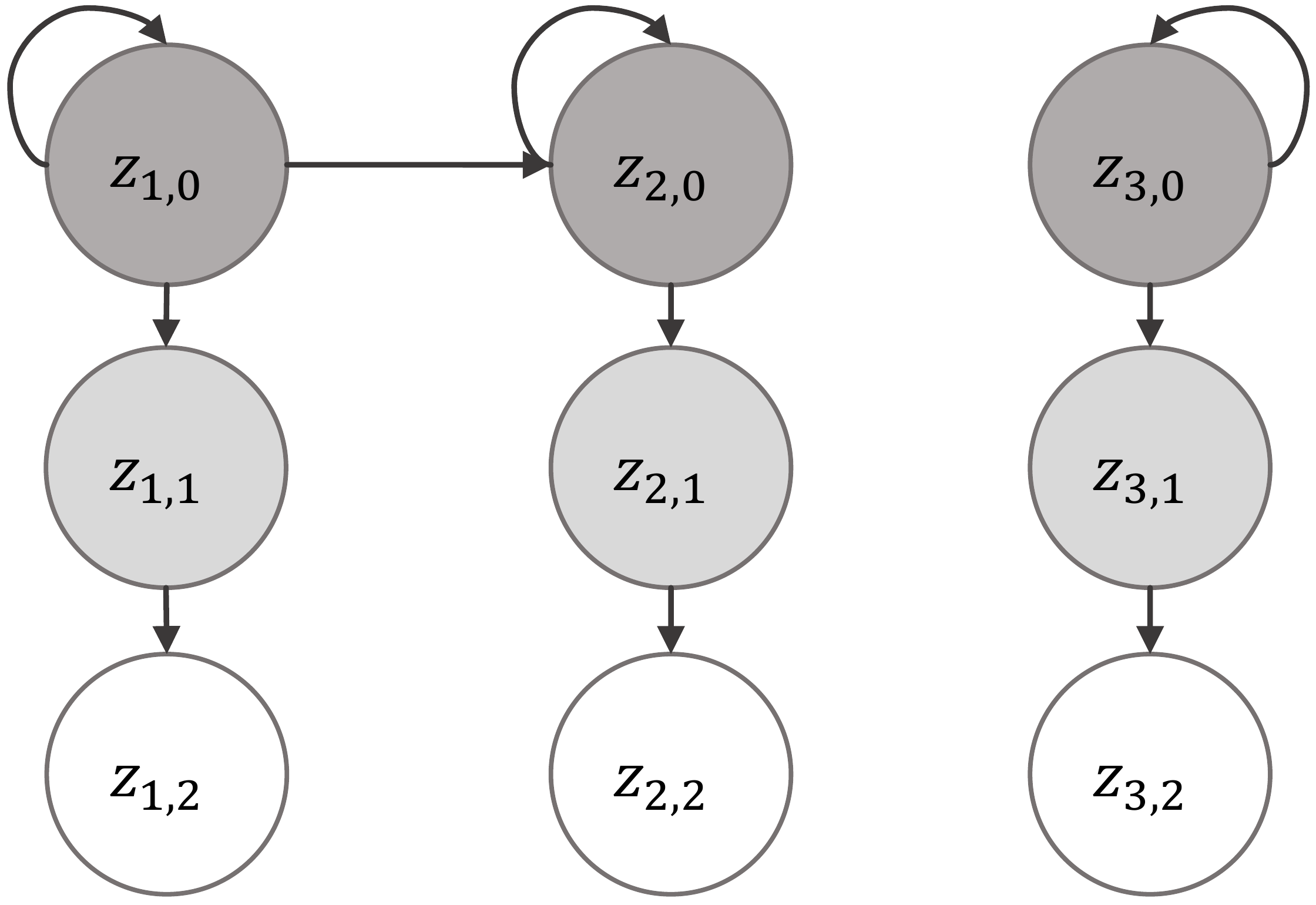}}
		\subfloat[]{\label{fig1-sec3}\includegraphics[width=0.22\linewidth]{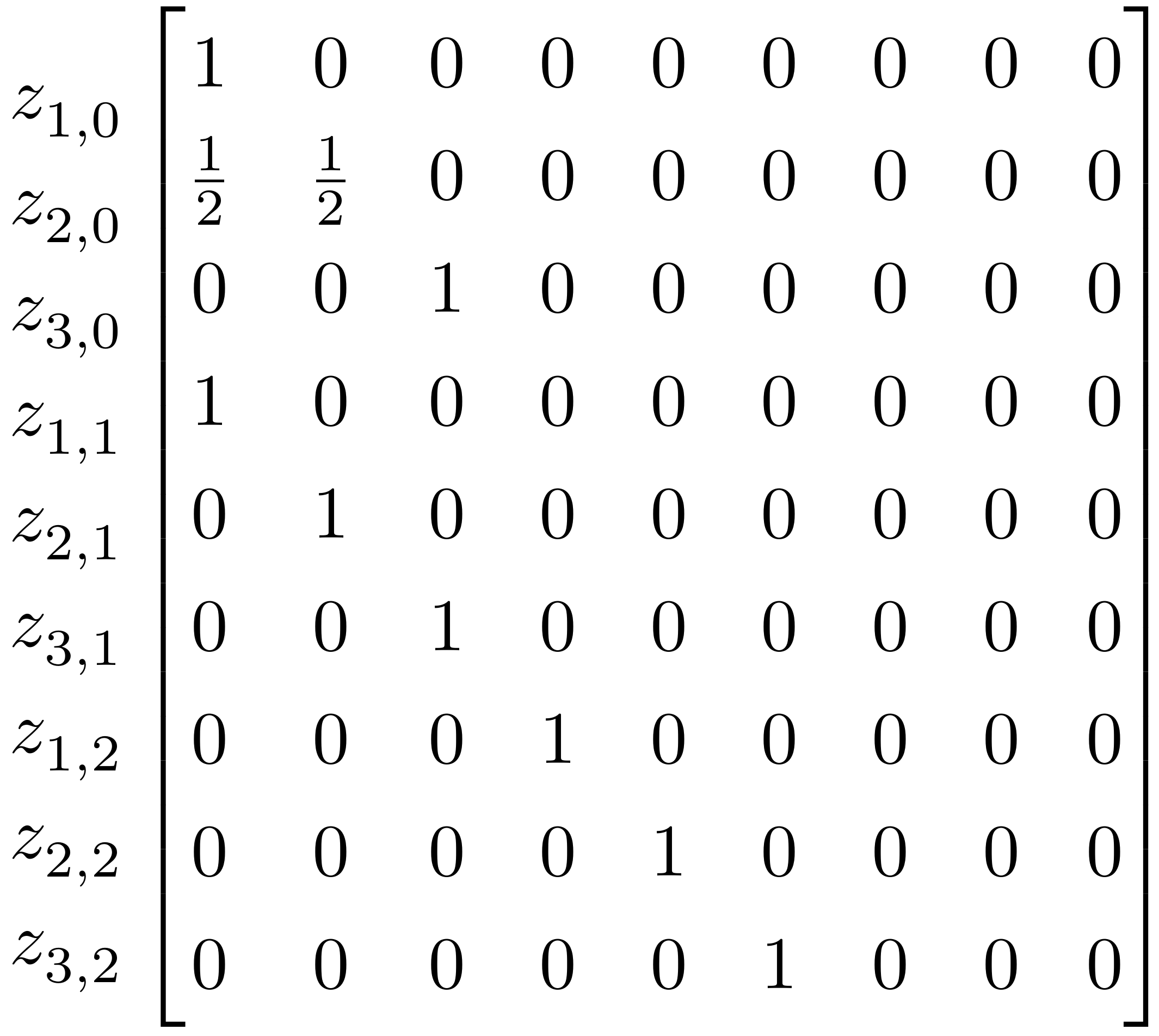}}
		\subfloat[]{\label{fig1-sec4}\includegraphics[width=0.28\linewidth]{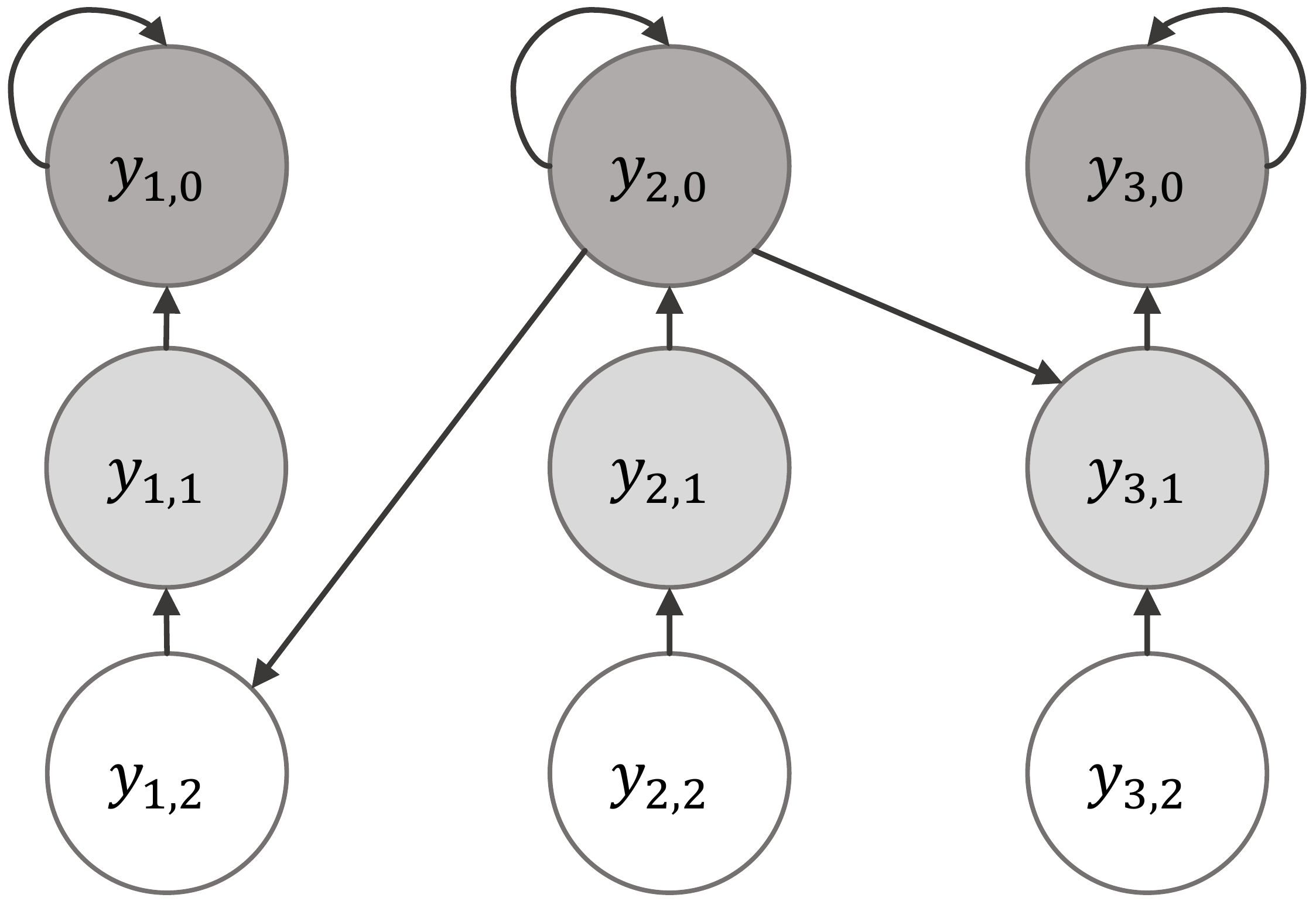}}
		\subfloat[]{\label{fig1-sec5}\includegraphics[width=0.22\linewidth]{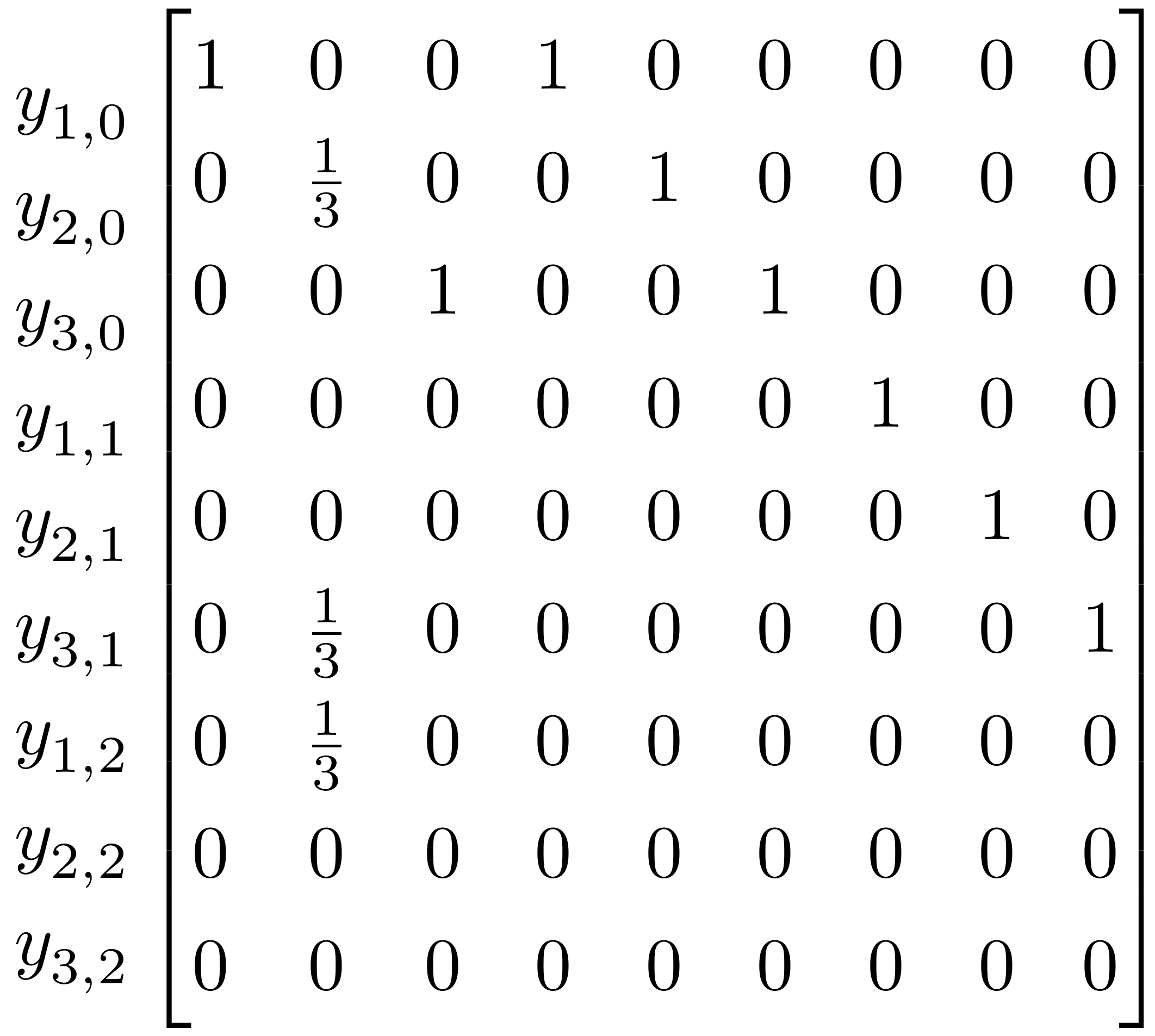}}
	\caption{Consider an example with $n=3, b=2$ with the original graph in (a). Suppose at $t^k$, the activated node $2$ uses $\vec z_2^{k-1}$ from itself and $\vec z_1^{k-1}$ from node $1$'s latest update
	to compute and send out $\vec z_2^k$ and $\vec y_2^k$, which are later used by node $3$ and $1$ at $t^{k+2}$ and $t^{k+3}$, respectively. 
	 (b) and (d) illustrate $( {\mathcal{V}}_z,  {\mathcal{E}}_z^k)$ and $( {\mathcal{V}}_y,  {\mathcal{E}}_y^k)$ at time $t^k$.
	 (c) and (e) provide the corresponding $ {\bm{H}}_{\sR}^k$ and $ {\bm{H}}_{\sC}^k$. Clearly, $ {\bm{H}}_{\sR}^k$ and $ {\bm{H}}_{\sC}^k$ are row- and column-stochastic matrices, respectively.
	}
	\label{fig1-sec}
\end{figure}

In this section, we develop a push-pull augmented graph approach to interpret Algorithm \ref{alg:proposed} via a single timescale, which is essential to the proof of Theorem \ref{main_theo} in Appendix. The graph augmentation is firstly adopt in \cite{nedic2010convergence} for consensus problems and then extended for distributed optimization problems in  \cite{zhang2019asyspa, assran2021asynchronous}. However, they only use one type of virtual nodes for each node of $\cG$, which is called the {\em pull} augmented digraph in this work. Also different from \cite{tian2020achieving} with augmented edges, we further design a {\em push} augmented digraph. Under such a push-pull augmented scheme, the push digraph aims to address $\vz_i$ while pull digraph is to address $\vy_i$ in Algorithm \ref{alg:proposed}.

\subsection{The push/pull augmented digraphs}

For each node $i$ of $\cG$, we reindex it as $ {z}_{i,0}$ and introduce $b$ virtual nodes $ {z}_{i,u}$, $u=1,\ldots, b$, where $b$ is given in Assumption \ref{assum1}(b). If node $i$ starts to update at time $t^k\in \cT$, we construct a {\em pull} augmented digraph $ {\mathcal G}_z^k= ( {\mathcal V}_z, {\mathcal E}_z^k)$ to address the asynchronicity and delays of $\vz_i$ in the proof, where $ {\mathcal V}_z=\{ {z}_{i, u}|1\le i\le n,0\le u\le b\}$ and $ {\mathcal E}_z^k$ includes all the directed edges $( {z}_{i, u},  {z}_{i, u+1})$ for $0 \le u\le b-1$.  In such an update, if node $i$ uses $\vz_j^{k-u}$ from node $j\in\cN_{in}^i$, then $( {z}_{j, u-1}, {z}_{i, 0})\in  {\mathcal E}_z^k$.  Note that $u\le b-1$,
and node $i$ may use multiple receptions from node $j$, which implies that more than one nodes in $\{ {z}_{j, u}|0\le u\le b\}$ are directly linked to $ {z}_{i, 0}$.  Moreover, each node of $ {\mathcal G}_z^k$ pulls information from its in-neighbors and uses their \emph{average} to update its state, which follows from the average operator in \eqref{gg4}.

To address $\vy_i$ in Algorithm \ref{alg:proposed}, we further construct a push 
augmented digraph $ {\mathcal G}_y^k= ( {\mathcal V}_y, {\mathcal E}_y^k)$ where the striking difference from $ {\mathcal G}_z^k$ lies in the information direction. Particularly, $ {\mathcal E}_y^k$ includes all the directed edges $( {y}_{i, u+1},  {y}_{i, u})$ for $0 \le u\le b-1$ and if node $i$ uses $\vy_j^{k-u}$ from node $j\in\cN_{in}^i$ at $t^k\in\cT$, then $({y}_{j, 0},  {y}_{i, u-1})\in  {\mathcal E}_y^{k-u}$.  To the contrary of the pull digraph, each node of $ {\mathcal G}_y^k$ pushes its latest state to its out-neighbors that are defined over $ {\mathcal E}_y^k$ and uses its in-neighbors' \emph{sum} to update its state, see the summation operator in \eqref{gg4}.

\subsection{Reformulation of Algorithm \ref{alg:proposed} via a single timescale}

With the aid of the push/pull digraphs, we are able to rewrite Algorithm \ref{alg:proposed}  via a single timescale.  To this end, let $\vec z_{i,u}^{k},\vec y_{i,u}^{k},1\le i\le n,0\le u\le b$ denote the latest state of nodes $ {z}_{i, u},  {y}_{i, u}$ before $t^{k+1}$.
Let $\mathcal T_i$ be an increasing sequence to record update time instants of node $i$.
Though it is tedious, one can follow our previous \cite{zhang2019asyspa} to show that there exist a row-stochastic matrix $ {\bm{H}}_{\sR}^k$ and a column-stochastic matrix $ {\bm{H}}^k_{\sC}$
such that Algorithm \ref{alg:proposed} is given by
\begin{align}\label{syn-update-a}
	{\bm{Z}}^{k+1}&= {\bm{H}}_{\sR}^k( {\bm{Z}}^k-\eta \bm{I}_a^k {\bm{Y}}^k),\\
\label{syn-update-b}
{\bm{Y}}^{k+1}& =  {\bm{H}}_{\sC}^k {\bm{Y}}^k+ \partial^{k+1}-\partial^k,
\end{align}
where
\bea\label{original_update}
{\bm{Z}}^{k}&=[\bm{Z}_0^{k};\ \ldots; \bm{Z}_n^{k}]\in\bR^{\tilde n\times 2d},\ 
&\bm{Z}_u^{k}&=[\vec z_{1,u}^{k};\ldots; \vec z_{n,u}^{k}]\bm{\Lambda}^{-1} \in\bR^{n\times 2d},\\
{\bm{Y}}^{k}&=[\bm{Y}_0^{k}; \ldots; \bm{Y}_n^{k}]\in\bR^{\tilde n\times 2d},\ 
&\bm{Y}_u^{k}&=[\vec y_{1,u}^{k}; \ldots; \vec y_{n, u}^{k}]\bm{\Lambda}  \in\bR^{n\times 2d},\\
\partial^k&=[\partial J^k; \ \bzero_{bn\times 2d}]\in\bR^{\tilde n\times 2d},\ 
&\partial J^k&=[\partial J_1^k; \ldots; \partial J_n^k]\bm{\Lambda} \in\bR^{n\times 2d},\\
\bm{\Lambda} &= [\bm{I}_d, \bm{0}; \bm{0}, \sqrt{\zeta} \bm{I}_d],
\ena
$$\begin{aligned}
\left[ {\bm{H}}_{\sR}^k\right]_{ij} &=
\left\{\begin{array}{ll}
	\frac{1}{|\mathcal{Z}_i^{k}|}, &  \text{if $j=nu+v, i,v\in\cV$,}
	\text{ $t^{k+1}\in\cT_i$,} \text{$i$ receives ${\vec z}_v^{k-u}$ at $t^{k+1}$},\\
	1,                            & \text{if $i\in\cV$, $t^{k+1}\notin\cT_i$, $j=i$,}                                                                     \\
	1,                            & \text{if $i\notin\cV$ and $j=i-n$,}                                                                                     \\
	0,                            & \text{otherwise,}
\end{array}\right.\\
\end{aligned}$$
$$\begin{aligned}
\left[ {\bm{H}}_{\sC}^k\right]_{ji} &=
\left\{\begin{array}{ll}
	\frac{1}{|\cN_{out}^i|},&  \text{if $j=nu+v,i,v\in\cV$,}
	\text{ $t^{k+1}\in\cT_i$,}\ \text{$v$ receives ${\vec y}_i^k$ at $t^{k+u}$},\\
	1,                            & \text{if $i\in\cV$, $t^{k+1}\notin\cT_i$ and $j=i$,}                                                                        \\
	1,                            & \text{if $i\notin\cV$ and $j=i-n$,}                                                                                        \\
	0,                            & \text{otherwise,}
\end{array}\right.\\
\end{aligned}$$
where $|\mathcal{Z}_i^{k}|$ is the number of receptions in buffer $\mathcal{Z}_i$ at time $t^{k+1}$. Moreover,
\bee\label{iak}
\begin{aligned}
\left[\bm{I}_a^k\right]_{ij}=\left\{\begin{array}{ll}
	1, & \text{if $i=j$, $i\in\cV$, } \text{and $t^{k+1}\in\cT_i$},\\
	0, & \text{otherwise,}
\end{array}\right.
\end{aligned}
\ene
and $ {\bm{Z}}^0=[\bm{Z}_0^0; \bzero_{2m\times bn}] $, $ {\bm{Y}}^0=[\partial J^0; \bzero_{2m\times bn}]$. 
	
\subsection{The distributed SAG understanding of Algorithm \ref{alg:proposed}}	
		The major difference from \cite{zhang2019asynchronous}  lies in the novel use of the SAG. Note that their work does not consider the SG issues. Recall from \eqref{sumy} that the summation of $\vy_i$ aims to track the scaled SAG of $J(\cdot)$. Now, we confirm this finding in details. 
		
		If $t^{k}\in\cT_i$ and node $i$ has selected the sample $p_i$ in Algorithm \ref{alg:proposed}, let $\tau_{i,p_i}^{k}=k$, which is the latest iteration index for using the sample $p_i$. Accordingly, define
		\bea\label{tau-define}
		\left\{\begin{array}{ll}
			\tau_{i,p}^{k}=k, & \text{if $t^{k}\in\cT_i$, and $p=p_i$},\\
			\tau_{i,p}^{k}=\tau_{i,p}^{k-1}, & \text{otherwise.}
		\end{array}\right.
		\ena
		and $\tau_{i,p}^{0}=0$ for all $p\in\cM_i$.
		Jointly with \eqref{original_update}, it implies that
		\bea\label{aggre-define}
		\partial J_i^k = \frac1{{m}}\sum_{p=1}^{m_i}\nabla J_{i,p}\left(\vec z_{i}^{\tau_{i,p}^k}\right).
		\ena
If $t^{k+1}\in\cT_i$, it is clear from \eqref{iak} that
$$\partial^{k+1}-\partial^k=\bm{I}_a^k(\partial^{k+1}-\partial^k).$$
By \eqref{tau-define} and \eqref{aggre-define}, the $i$-th row of $\partial^{k+1}-\partial^k$ is
$$\begin{aligned}
\partial J^{k+1}_i-\partial J^k_i = \frac1{{m}} \sum_{p=1}^{m_i} (\nabla J_{i,p}(\vec z_{i}^{\tau_{i,p}^{k+1}})-\nabla J_{i,p}(\vec z_{i}^{\tau_{i,p}^{k}}) )
= \frac1{{m}} (\nabla J_{i,p}(\vec z_{i}^{k+1}) - \nabla J_{i,p}(\vec z_{i}^{\tau_{i,p}^{k}})).
\end{aligned}$$

It follows from the initialization of Algorithm \ref{alg:proposed} that $\vec y_{i,0}^0=\partial J_i^0=\frac1{{m}}\sum_{p=1}^{m_i}\nabla J_{i,p}(\vec z_{i}^{0})$. Left-multiplying $\bone^\T_{\tilde{n}}$ on both sides of \eqref{syn-update-b} leads to that
\bea\label{lemma-tracking}
\bone^\T_{\tilde{n}} {\bm{Y}}^{k}=\bone^\T_{\tilde{n}}\partial^{k}=\frac1{{m}}\sum_{i=1}^n\sum_{p=1}^{m_i}\nabla J_{i,p}\left(\vec z_{i,p}^{\tau_{i,p}^k}\right).
\ena
By \eqref{original_update} and the construction of push/push digraphs, it follows that $\bone^\T_{\tilde{n}} {\bm{Y}}^{k}=\sum_{i=1}^n \vy_i^k$, if there is no transmission delays between nodes.

\section{Numerical results}\label{bigsec4}
\subsection{Experiment for MARL}\label{sec4-1}
We compare APP-SAG with PD-DistIAG \cite{wai2018multi} and FDPE \cite{cassano2020multi} to solve a real networked MARL problem similar to \cite{qu2020scalable}, with 9 users/agents and 4 access points involving, see Fig. \ref{fig2-sec1}. 
Each agent can communicate with access points on the corner of its square area.
One agent keeps at most one information packet. 
At time $t$, agent $i$ receives a new packet with probability $p_{1,i}$.
The agent can then choose to keep it, with success probability $p_{2,i}$, or to send it to one of the available access points $l$. If simultaneously no other agents send to $l$, then the transmission will complete with probability
$q_l$ (dependent on $l$) and $i$ will get a local reward of $1$. Otherwise, if multiple agents send to the same access point $l$, all their transmissions will fail. 
In this experiment, we assume that the state of a single agent, $s_i\in \{0,1\}$ ($s_i$ is the package amount of the agent $i$), is observable to all the agents, and then the MDP obtains $2^9=512$ states. 

\begin{figure*}[!t]
	\centering
	\captionsetup{justification=raggedright}
	\subfloat[]{\label{fig2-sec1}\includegraphics[width=0.31\linewidth]{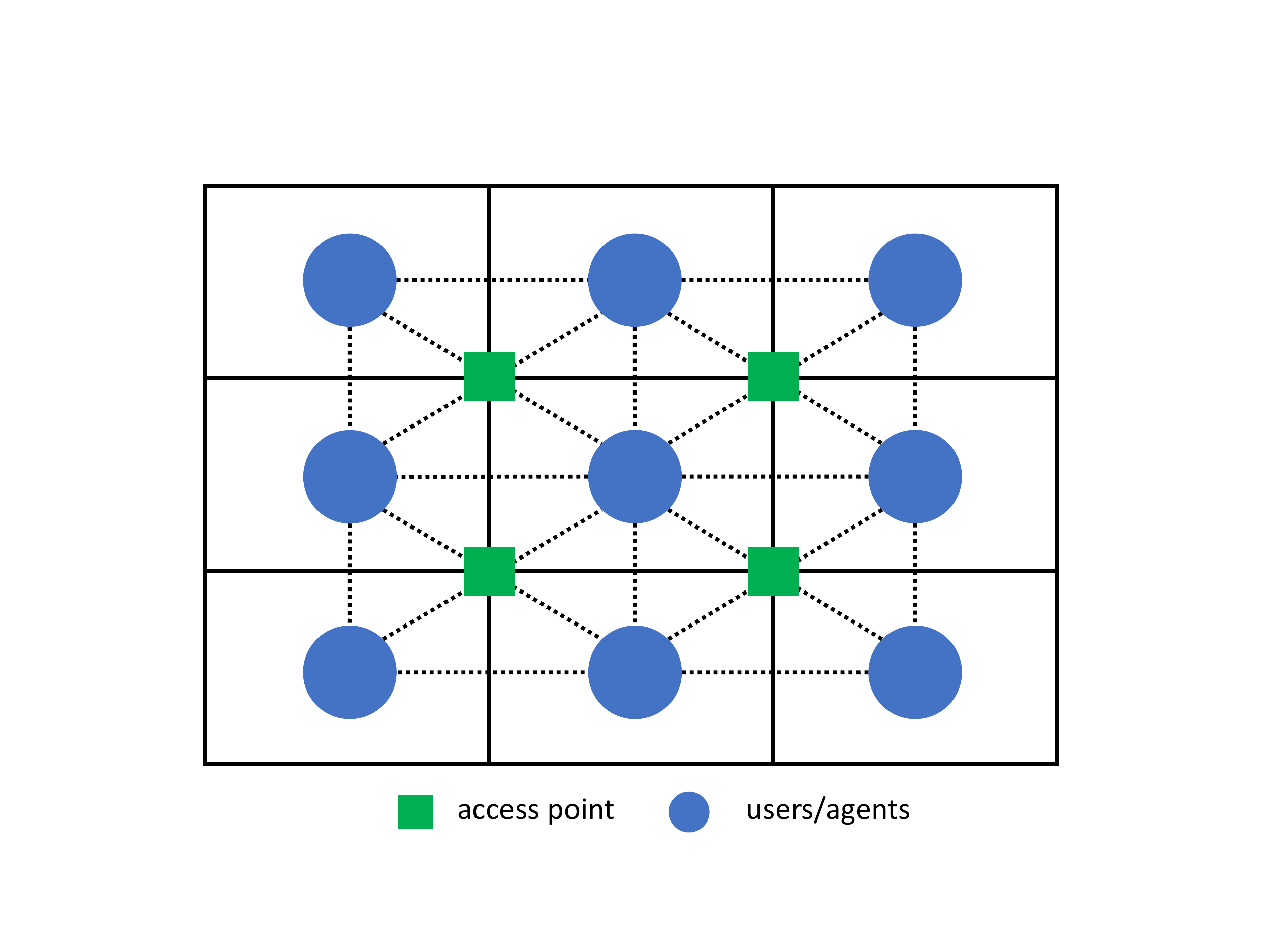}}
	\subfloat[]{\label{fig2-sec2}\includegraphics[width=0.33\linewidth]{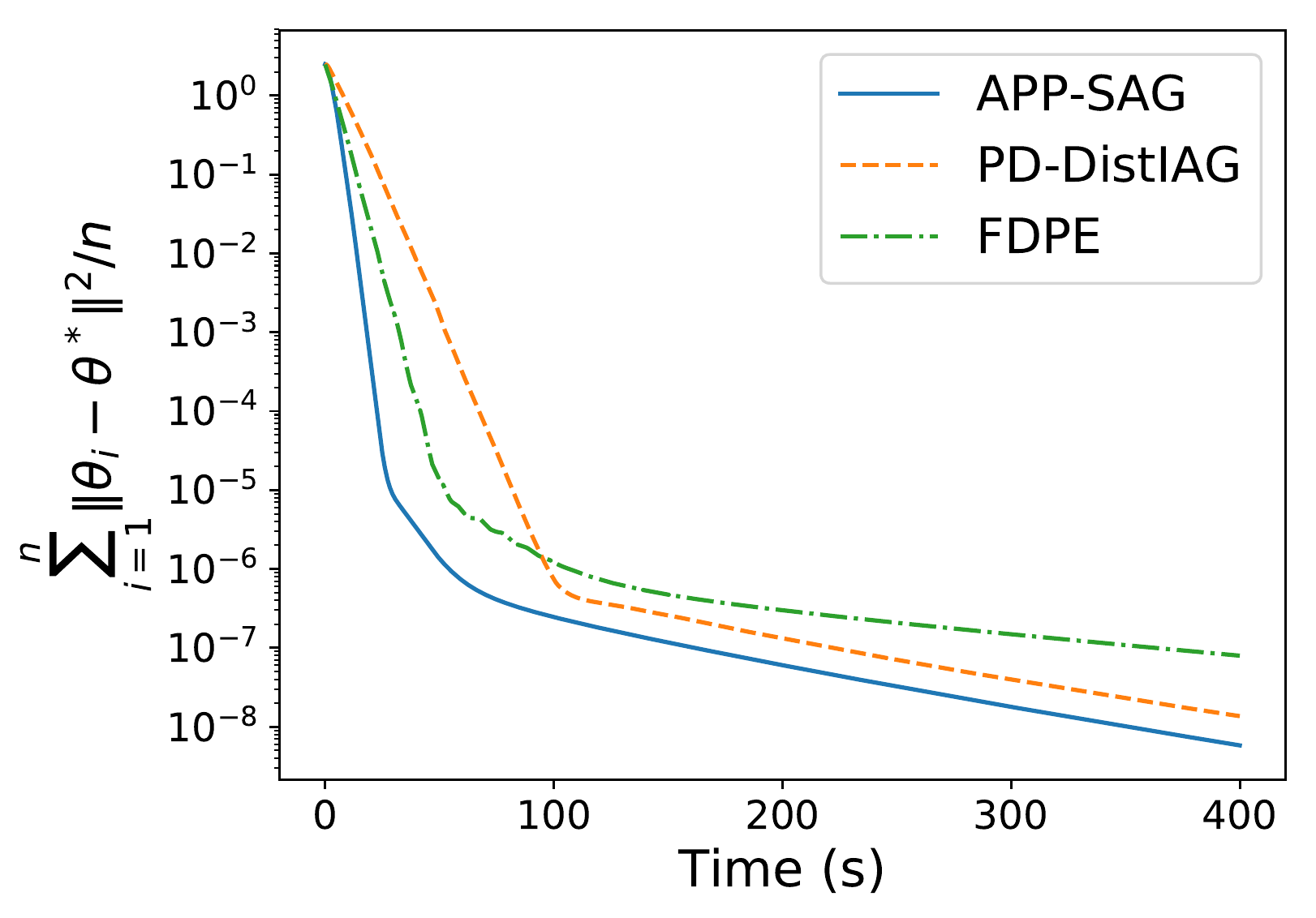}}
	\subfloat[]{\label{fig2-sec3}\includegraphics[width=0.33\linewidth]{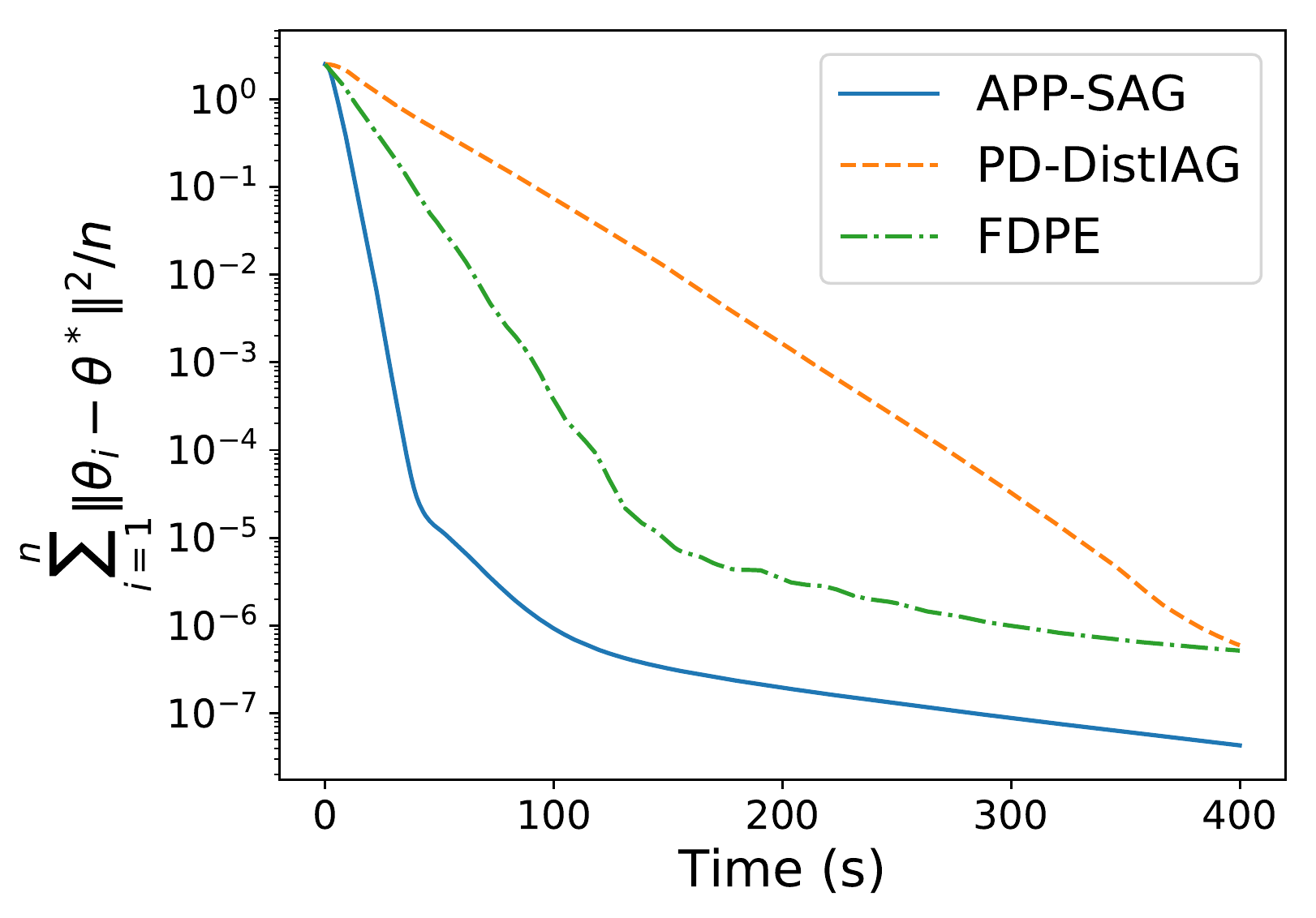}}
	\caption{(a) User-access problem with nine nodes and four access points, where dot lines connect pairs of neighbors in the undirected network. (b) Convergence performance of APP-SAG, PD-DistIAG and FDPE with similar cores. (c) 
	Convergence performance of APP-SAG, PD-DistIAG and FDPE with one core slowed down.}
	\label{fig-sec}
\end{figure*}

\begin{figure}[!t]
	\centering
	\includegraphics[width=0.4\linewidth]{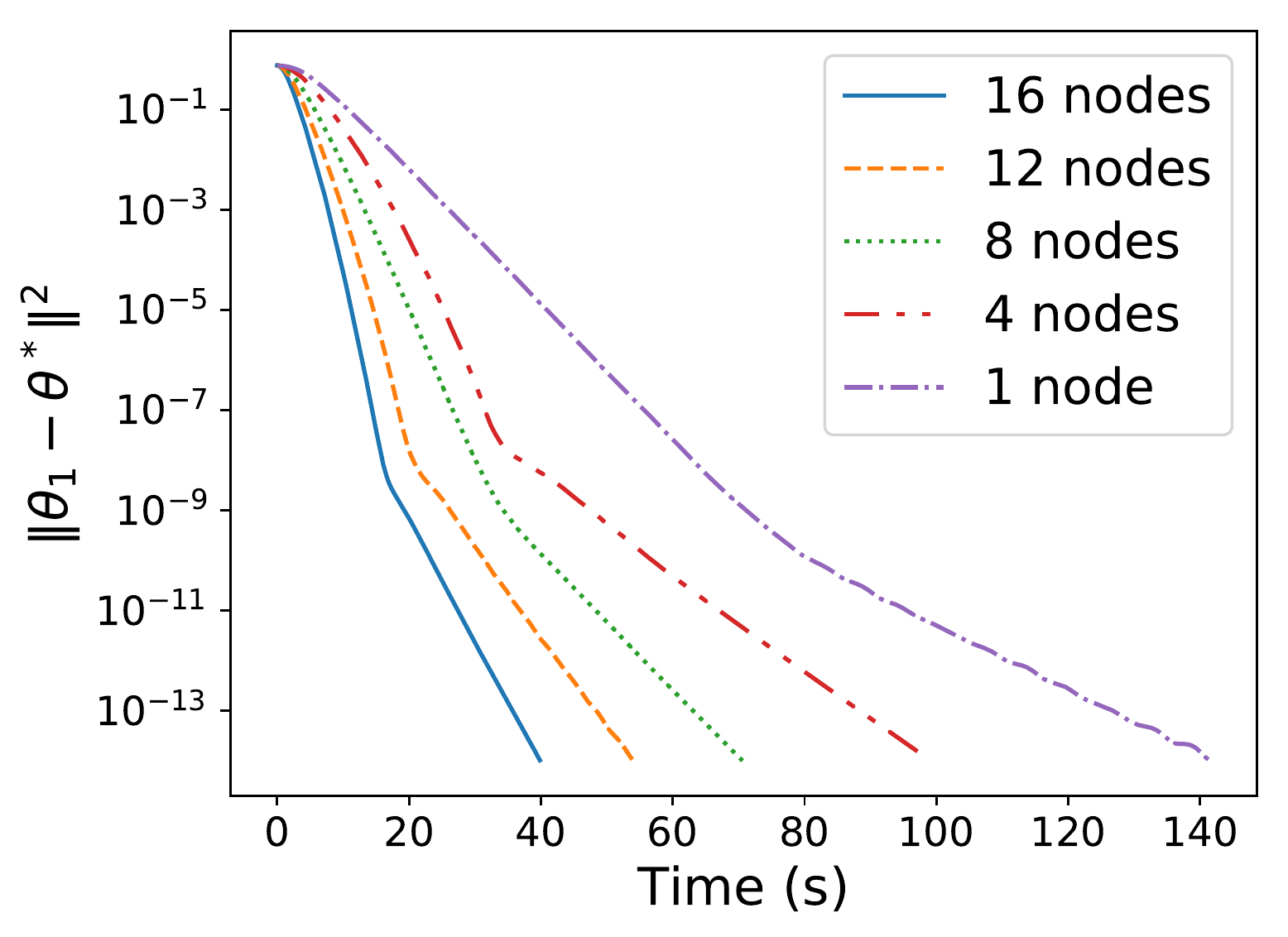}
	\caption{Speedup of APP-SAG. The speed for reaching a fixed error level increases approximately linearly with more nodes computing.}
	\label{fig-speedup}
\end{figure}

We adopt a simple Sarsa algorithm with feature number $d=10$ to learn a fixed policy. With the policy, we obtain trajectory experiences of states, actions and rewards of sample size $m_i=5000$.
We compare APP-SAG with PD-DistIAG and FDPE based on \texttt{OpenMPI} with nine CPU cores playing the roles of nodes. The network topology is undirected, see Fig. \ref{fig2-sec1}. Note that PD-DistIAG and FDPE only work on undirected networks. The batch size is $1$, and the stepsizes for primal/dual vectors of APP-SAG, PD-DistIAG and FDPE
are $[\eta_1,\eta_2]=[1.4\times10^{-5},2.9\times10^{-4}]$, $[\eta_1,\eta_2]=[1.8\times10^{-5},8\times10^{-4}]$, and $[\eta_1,\eta_2]=[1.1\times 10^{-5}, 3.8\times10^{-4}]$, respectively, which are manually tuned for the best performances.

In Fig. \ref{fig2-sec2}, the performance of APP-SAG is consistent with our theoretical result and confirms the linear convergence to the optimal solution.
Then we slow down one fixed core (the upper-right one in Fig. \ref{fig2-sec1}) for both algorithms. Fig. \ref{fig2-sec3} demonstrates that APP-SAG maintains its performance while PD-DistIAG and FDPE suffer greatly from the slow core. 

\subsection{Experiment for parallel RL}

For parallel RL, we examine the speedup effect of APP-SAG, i.e., the more nodes the better convergence speed. On {mountaincar} task from \cite{sutton1998introduction},
the data with feature number $d=30$ and volume ${m}=24000$, are unevenly divided to $n$-node networks ($n=1,4,8,12,16$), with data volumes proportional to agent serial numbers. In the four-agent scenario, agent $1, 4$ takes $1/10$
and $4/10$ of the total transition samples, respectively.
In the network topology, each node $i$ sends information to node $\mbox{mod}(2^j+i,n)$, where $0\le j < [\log_2(n)]$. The topology is directed with a connectivity of $\mathcal{O}(\log(n))$. Note that FDPE does not work in this directed network.

We set the local batch size as $64/n$, so that the problem has a fixed computational workload for networks with different numbers of nodes.
The stepsizes are tuned to empirically perform best for all the node numbers, $1.1\times 10^{-4}, 9.6\times10^{-5}, 1.3\times10^{-4}, 1.5\times10^{-4}, 1.9\times10^{-4}$ for $\eta_1$ of node number $16, 12, 8, 4, 1$, respectively,
while $\eta_2/\eta_1=\zeta\approx7.0$ are similar among experiments for different node numbers.
Fig. \ref{fig-speedup} shows the convergence curves of node $1$ for APP-SAG under different node numbers, which verifies the APP-SAG's linear speed-up property.

\section{Conclusion}\label{bigsec5}
This work has proposed a novel fully asynchronous algorithm for policy evaluation of DisRL over a directed network.
The striking feature allows each node to communicate with its neighbors and update its local variables locally at any time.
From the worst-case view, we show that APP-SAG converges linearly with respect to a newly introduced virtual counter. 
Simulation results on MARL and parallel RL examples have illustrated the significance of APP-SAG.  Since \eqref{equ:2-3:primal-dual-problem} has a specific form, our future works will consider using APP-SAG to solve general primal-dual problems.

\appendix

\section{Proof of Theorem 1}\label{prooj_theo}

\subsection{Notations and preliminaries}\label{preliminaries}
Before proving Theorem \ref{main_theo}, we provide some important properties. 
We first introduce an important concept called 
\emph{absolute probability sequence}. \cite[Theroem 4.2]{touri2012product} claims that 
for any sequence of row-stochastic matrices $\{\bm{Q}^k\}_{k=1}^t, t\in \bN$, there exists a sequence of stochastic vectors $\{\vu^k\}_{k=1}^{t+1}$ that satisfies
$
(\vu^{k+1})^\T \bm{Q}^k=(\vu^k)^\T,
\ \forall k\in\bN
$. Note that $\{\vu^k\}_{k=1}^{t+1}$ depends on the whole sequence $\{\bm{Q}^k\}_{k=1}^t$. In \eqref{syn-update-a}, $\{ {\bm{H}}_{\sR}^k\}$ is a row-stochastic matrices sequence and let $\{\bm{u}^k\}$ be the absolute probability sequence, i.e., $(\bm{u}^{k+1})^\T  {\bm{H}}_{\sR}^k=(\bm{u}^k)^\T$.
We use $\bm{u}^k$ to track the average of ${\bm{Z}}^k$ in the proof.
\par We let $\vec v^{k+1}= {\bm{H}}_\sC^{k}\vec v^{k}$ and $\vec v^0=[\bone_n; \bzero_{\tilde n-n}]$. 
Let $\bm{V}^{k}=\diag(\vec v^{k})$ and 
$(\bm{V}^{k})^\dag$ be the Moore-Penrose inverse of $\bm{V}^k$, i.e.,
$$
\left[ (\bm{V}^{k})^\dag \right]_{ij}= \left\{\begin{array}{ll}
	1/\left[ \bm{V}^k \right]_{ii},\ &\mbox{if}\ i=j,\ \left[ \bm{V}^k \right]_{ii}>0, \\ 
	0,\ &\mbox{otherwise.}\\
	\end{array}\right.
$$
We further define 
\bea\label{S_V}
\bm{I}_\sV^k=\bm{V}^k(\bm{V}^k)^\dag,\ \bone_\sV^k=\bm{I}_\sV^k\bone_{\tilde n},\ 
 {\bm{Y}}_\sV^k=(\bm{V}^{k})^\dag {\bm{Y}}^k,
\ena
where $ {\bm{Y}}^k$ is defined in \eqref{original_update}. 
In the proof, we use $\vec v^{k}$ to track the average of ${\bm{Y}}^k$.

\par In the following lemma, we show that the accumulative products of $ {\bm{H}}_{\sR}^k$ and $ {\bm{H}}_{\sC}^k$ converge to rank-one matrices linearly.
\begin{lemma}\label{lemma00} (\cite[Lemma 2]{zhang2019asynchronous})
Define the accumulative product of $ {\bm{H}}_{\sR}^k$ and $ {\bm{H}}_{\sC}^k$ as
$$\begin{aligned}
\bm{\Phi}_\sR^{k:k+t}&= {\bm{H}}_{\sR}^{k+t-1}\ldots  {\bm{H}}_{\sR}^{k+1} {\bm{H}}_{\sR}^k,\quad\ 
\bm{\Phi}_\sC^{k:k+t}&= {\bm{H}}_{\sC}^{k+t-1}\ldots  {\bm{H}}_{\sC}^{k+1} {\bm{H}}_{\sC}^k,\\
\end{aligned}$$
where $k\ge 0,t > 0$. If $t=0$, define $\bm{\Phi}_\sR^{k:k}=\bm{\Phi}_\sC^{k:k}=\bm{I}_{\tilde n}$.
\par Under Assumption 1, the following statements hold.
\begin{enumerate}[(a)]
\item{
	For all $k,t\ge 0$, there exist stochastic vectors $\bm{\phi}_\sR^{k:k+t},\bm{\phi}_\sC^{k:k+t}\in \mathbb{R}^{\tilde n}$ such that
	$$\begin{aligned}
	&\Vert \bm{\Phi}_\sR^{k:k+t} - \bone_{\tilde n}(\bm{\phi}_\sR^{k:k+t})^\T \Vert_\sF \le 2\delta^t,\quad\ 
	&\Vert \bm{\Phi}_\sC^{k:k+t} - \bm{\phi}_\sC^{k:k+t}\bone_{\tilde n}^\T \Vert_\sF \le 2\delta^t,\\
	\end{aligned}$$
	where $\delta = (1-\kappa)^{1/d_g b}$ and $\kappa = \left(1/{\tilde n}\right)^{d_g b}\in (0,1)$, $d_g$ is the diameter of $\cG$, i.e., the largest distance between any pair of nodes in $\cG$.
}
\item{
	For all $i\in \cV, j\in {\cV}_y, t\ge 0$, $\sum_{j=1}^n \left[ \bm{\Phi}_\sC^{0:t} \right]_{ij} \ge  n\kappa$.
	
}
\end{enumerate}
\end{lemma}
Then, Corollary \ref{coro1} follows directly, by which we introduce the constant $\bt$ and $\mu$ to characterize the long-term 
convergence property of our method.
\begin{corollary}\label{coro1}
Under the conditions in Lemma \ref{lemma00}, define $\mu< \kappa/{2n}\in (0,1)$. There exists a constant $\bt\ge 0$ such that $\delta^{\bt}\le \mu/2$, and 
\bea\label{result-cor-14}
	&\Vert \bm{\Phi}_\sR^{k,k+\bt} - \bone_{\tilde n}(\bm{\phi}_\sR^{k,k+\bt})^\T \Vert_\sF \le \mu,\ 
	\Vert \bm{\Phi}_\sC^{k,k+\bt} - \bm{\phi}_\sC^{k,k+\bt}\bone_{\tilde n}^\T \Vert_\sF \le \mu, &\forall k\ge 0.
\ena
\end{corollary}
We define $\vec z=[\bm{\theta}; \bm{\omega}/\sqrt{\zeta}]$, and
	$$\begin{aligned}
	\bm{G}_{i,p}&=\frac{1}{{m}}\begin{bmatrix} \rho \bm{I} & -\sqrt{\zeta} \widehat{\bm{A}}_{i,p}^\T\\\sqrt{\zeta} \widehat{\bm{A}}_{i,p} & \zeta \widehat{\bm{C}}_{i,p}  \end{bmatrix},
	\\ \nabla j_{i,p}(\vec z)&=\frac{1}{{m}}[ \nabla_{\bm{\theta}} J_{i,p}(\bm{\theta},\bm{\omega}); -\sqrt{\zeta}\nabla_{\bm{\omega}} J_{i,p}(\bm{\theta},\bm{\omega})]
	=\frac{1}{{m}}[\widehat{\bm{A}}_{i,p}^\T\bm{\omega} + \rho \bm{\theta}; \sqrt{\zeta}(\widehat{\bm{A}}_{i,p}\bm{\theta}- \widehat{\bm{C}}_{i,p}\bm{\omega}-\widehat{\bm{b}}_{i,p})],
	\\ \bm{G}_i&=\sum_{p=1}^{m_i} \bm{G}_{i,p},\ \nabla j_i(\vec z)=\sum_{p=1}^{m_i} \nabla j_{i,p}(\vec z),\quad \ 
	\bm{G}=\sum_{i=1}^n \bm{G}_i,\ \nabla j(\vec z)=\sum_{i=1}^{n} \nabla j(\vec z).
	\end{aligned}$$
	We use the following lemma to study the properties of $j(\vec{z})$ and its corresponding gradients and SG gradients.
\begin{lemma}\label{lemma0-1}
	 If $\zeta = \eta_1/\eta_2,\eta_1=\eta$ satisfies that $\zeta>\frac{4\eta + 4\lambda_{\text{max}}(\widehat{\bm{A}}^\T\widehat{\bm{C}}^{-1}\widehat{\bm{A}})}{\lambda_{\text{min}}(\widehat{\bm{C}}^{-1})}$, and $\eta<\lambda_{\text{max}}(\bm{G})$,
	then there exists $\alpha = \lambda_{\text{min}}(\bm{G})> 0 $, such that for all $\vec z\in\bR^{2d}$,
	\bea\label{lem-alpha}
	\Vert \vec z-\eta \nabla j(\vec z)-\vec z^\star\Vert_2\leq (1-\alpha \eta)\Vert\vec z-\vec z^\star\Vert_2.
	\ena
	and for all $\vec z_1,\vec z_2 \in\bR^{2d}, \forall i, p,$
	\bea\label{lem-beta}
	\Vert\nabla j_{i,p} (\vec z_1) - \nabla j_{i,p} (\vec z_2)\Vert_2 &\le \beta\Vert \vec z_1 - \vec z_2 \Vert_2,\\
	\ena
	where $\beta=\max_{1\le i\le n, 1\le p \le m_i}\lambda_{\text{max}}(\bm{G}_{i,p})$.
	Define $\psi=\lambda_{\text{max}}(\widehat{\bm{C}})$, then
	\bee\label{bound-psi}
	\beta> \frac{\zeta\psi}{2m}.
	\ene
\end{lemma}

\begin{proof}
	According to \cite[Appendix A]{du2017stochastic}, if $\zeta>\frac{4\eta + 4\lambda_{\text{max}}(\widehat{\bm{A}}^\T\widehat{\bm{C}}^{-1}\widehat{\bm{A}})}{\lambda_{\text{min}}(\widehat{\bm{C}}^{-1})}$,
	$\bm{G}$ is a diagonalizable matrix with all eigenvalues of $\bm{G}$ are positive and real.
	By direct computation, we can verify that
	$$
	\vec z-\eta\nabla j(\vec z)-\vec z^\star = (\bm{I}-\eta \bm{G}) (\vec z-\vec z^\star),
	$$
	where
	$\varrho(\bm{I}-\eta\bm{G}) \le 1-\eta\lambda_{\text{min}}(\bm{G}).$
	Define $\alpha = \lambda_{\text{min}}(\bm{G})$ and \eqref{lem-alpha} follows.
	\par By computation, for all $\vec z_1,\vec z_2 \in\bR^{2d},$
	$
	\nabla j_i (\vec z_1) - \nabla j_i (\vec z_2) = \bm{G}_i (\vec z_1 - \vec z_2),
	\nabla j_{i,p} (\vec z_1) - \nabla j_{i,p} (\vec z_2) = \bm{G}_{i,p} (\vec z_1 - \vec z_2),
	$
	and \eqref{lem-beta} follows. 
	
	Let $\vx\in \bR^d$ be an arbitrary eigenvector of $\widehat{\bm{C}}$ that $\widehat{\bm{C}}\vx=\lambda_{\sC}\vx$. Note that 
	$$\bm{G}=\begin{bmatrix} \rho \bm{I} & -\sqrt{\zeta} \widehat{\bm{A}}^\T \\ \sqrt{\zeta} \widehat{\bm{A}} & \zeta \widehat{\bm{C}}  \end{bmatrix},$$
	then $[\vx;\vx]^\T\bm{G}[\vx;\vx]=\frac{1}{{m}}(\rho +\zeta\lambda_{\sC})\vx^\T\vx$. By diagonalizability, $[\vx;\vx]^\T\bm{G}[\vx;\vx]\le 2\lambda_{\text{max}}(\bm{G})\vx^\T\vx$, and then 
	$\beta\ge\frac{1}{m}\lambda_{\text{max}}(\bm{G})>\frac{1}{2{m}}\zeta\lambda_{\sC}$ holds for all the eigenvalues of $\widehat{\bm{C}}$.
	Thus, \eqref{bound-psi} is in force.
\end{proof}

\subsection{Proof Sketch of Theorem \ref{main_theo}}\label{brief-proof}

We define the following quantities.
\begin{enumerate}[(a)]
	\item{
		$\|\check{\bm{Z}}^k\|_\sF=\|\bm{T}_\sU^k {\bm{Z}}^k\|_\sF$, where $\bm{T}_\sU^k=\bm{I}_{\tilde n}-\bone_{\tilde n}(\vu^{k-1})^\T$,
		which is the weighted consensus error of $ {\bm{Z}}^k$ in the augmented network.	
	}
	\item{
		$\Vert \check{\bm{Y}}_\sV^k \Vert_\sF=\|\bm{T}_\sC^k {\bm{Y}}_\sV^k\|_\sF$, where $\bm{T}_\sC^k=\bm{I}_{\sV}^k-\frac1n \bone_\sV^k (\vec v^k)^\T$, which is an error estimate corresponding to the gradient surrogates.
	}
	\item{
		$\|\check{\vec z}_\sU^k\|_\sF$, where $\check{\vec z}_\sU^k = \vec z_\sU^k - \vec z^*
		=( {\bm{Z}}^k)^\T\vu^k - \vec z^*$, which is the optimality gap between the weighted average and the saddle point.
	}
\end{enumerate}

Then, $\check{\bm{Z}}^k$, $\check{\bm{Y}}_\sV^k$, $\check{\vec z}_\sU^k$ and $ {\bm{Y}}^k$
are bounded for all $k\ge 0$, i.e., 
	\begin{align}\label{ineq1}
		&\Vert\check{\bm{Z}}^{k+\bt}\Vert_\sF 
		\leq 2\mu \Vert\check{\bm{Z}}^k\Vert_\sF  +\eta \sqrt{2\tilde n}\sum_{t=0}^{\tilde{t}-1}\Vert {\bm{Y}}^{k+t} \Vert_\sF,\\
	\label{ineq2}
	&\Vert\check{\bm{Y}}_\sV^{k+\tilde{t}} \Vert_\sF \le
	2\kappa^{-1}\mu {n} \Vert \check{\bm{Y}}_\sV^k \Vert_\sF+ 2\beta  \kappa^{-1}\sqrt{\tilde n} bK\sum_{t=-bK}^{\tilde{t}-1}(2 \Vert\check{\bm{Z}}^{k+t}\Vert_\sF+ \eta \Vert {\bm{Y}}^{k+t}\Vert_\sF),\\
	\label{ineq3}
	&\Vert\check{\vec z}_\sU^{k+b} \Vert_\sF 
	\leq (1- \eta \alpha bn)\Vert\check{\vec z}_\sU^k\Vert_\sF+{\eta n} \sum_{t=0}^{ b-1}\Vert\check{\bm{Y}}_\sV^{k+t}\Vert_\sF+ 3\eta n\beta\sum_{t=0}^{b-1}\Vert\check{\bm{Z}}^{k+t}\Vert_\sF \\
	\notag
	&\quad+ \eta n bK \beta \sum_{t=-bK}^{b-1}\left(2\sqrt{n}\Vert\check{\bm{Z}}^{k+t}\Vert_\sF + \eta\Vert {\bm{Y}}^{k+t}\Vert_\sF\right),\\
	\label{ineq4}
	&\Vert {\bm{Y}}^k\Vert_\sF \le n \Vert \check{\bm{Y}}_{\sV}^k \Vert_\sF + 
	2\beta\sqrt{\tilde{n}n} \sum_{t=k-bK}^{k}\left\Vert \check{\bm{Z}}^t\right\Vert_\sF+ \beta\sqrt{\tilde{n}n} \sum_{t=k-bK}^{k}\left\Vert \vec z_\sU^t-\vec z^*\right\Vert_\sF,
	\end{align}
where $\tilde t$ and $\mu$ are given in Corollary \ref{coro1}, $\kappa$ is as introduced in Lemma \ref{lemma00},
$\alpha, \beta$ are in Lemma \ref{lemma0-1}, and $b,K$ from Assumption \ref{assum1}(b), \ref{assum2}(b). Any value with $k<0$ is regarded as $0$. The above inequalities will be formally proved in Appendix \ref{Appendix-ineq}.

\par To prove the convergence of the inequality linear system \eqref{ineq1}-\eqref{ineq4}, we introduce $\lambda$-sequence from \cite{nedic2017achieving}. 
Then, the proof of Lemma \ref{lemma-17} directly follows the proof of \cite[Lemma 8]{zhang2019asynchronous}.

\begin{definition}\label{lambda-sequence}
Let $\{\bm{p}^t\}$ be a nonnegative sequence with $\lambda \in (0,1)$, the $\lambda$-sequence of $\bm{p}^k$ be
$
\bm{p}^{\lambda,k}=\sup_{1\le t\le k} \frac{\bm{p}^t}{\lambda^t}.
$
If $\bm{p}^{\lambda,k}$ is bounded by $\bar{p}$ for all $k$, then $\bm{p}^k\le \bar{p}\lambda^k$.
\end{definition}

\begin{lemma}\label{lemma-17}
	Let $\{\vec p^k\},\{\vec q^k\}$ be nonnegative sequences satisfying
	$
	\vp^{t+j+l}\leq r\vp^{t+l}+\sum_{i=0}^{j+l-1}\vec q^{t+i},
	$
	where $r\in[0,1)$ is a scalar. If we choose $\lambda$ such that $\lambda^j\in(r,1)$, then the $\lambda$-sequences $\vp^{\lambda,k}$ and $\vec q^{\lambda,k}$ satisfy
	$
	\vp^{\lambda,k}\leq\frac{j+l}{\lambda^j-r}\vec q^{\lambda,k}+c_{\lambda},\ \forall k\in\bN,
	$
	where $c_{\lambda}=\frac{\lambda^j}{\lambda^j-r}\sum_{t=1}^{m}\lambda^{-t}\vp^t$ is a constant not related to $k$.
\end{lemma}

\par 
Based on Lemma \ref{lemma-17} and \eqref{ineq1}-\eqref{ineq4}, we can derive the following relationship between the $\lambda$-sequence corresponding to the above quantities. For all $k \ge 1$, if 
$\lambda = \max\left\{\sqrt[\tilde{t}+1]{\frac12 + \kappa^{-1}\mu n}, \sqrt[b+1]{1 - \eta \alpha\kappa n/2} \right\}$,
the following inequalities hold,
\bea
\notag
&\Vert\check{\bm{Z}}\Vert_\sF^{\lambda,k} \le \frac{2\eta n^2\sqrt{2\tilde{n}}\tilde{t}}{\kappa^2(n-1)} \Vert {\bm{Y}} \Vert_\sF^{\lambda,k} + p_1,\\
&\Vert \check{\bm{Y}}_\sV \Vert_\sF^{\lambda,k} \le \frac{4\beta\sqrt{\tilde n}{\tilde t}bK}{\kappa^2(1-\kappa)}\left( 2\Vert\check{\bm{Z}}\Vert_\sF^{\lambda,k} +  \eta\Vert {\bm{Y}} \Vert_\sF^{\lambda,k}\right) + p_2,\\
&\Vert\check{\vec z}_\sU\Vert_\sF^{\lambda,k} \le 
\frac{1}{(1-\eta\alpha\kappa n)\alpha\kappa} \left((6nb\beta+8b^2K^2\beta\sqrt{n}) \Vert\check{\bm{Z}}\Vert_\sF^{\lambda,k}+2b\Vert \check{\bm{Y}}_\sV\Vert_\sF^{\lambda,k}\right.
\left. + 2\eta b^2K^2\beta \Vert {\bm{Y}}\Vert_\sF^{\lambda,k}\right) + p_3,\\
&\Vert {\bm{Y}}\Vert_\sF^{\lambda,k} \le 2\beta\sqrt{n\tilde{n}}bK\Vert\check{\bm{Z}}\Vert_\sF^{\lambda,k} + n\Vert\check{\bm{Y}}_{\sV}^k\Vert_\sF^{\lambda,k}
+\beta\sqrt{n\tilde{n}}bK\Vert\check{\vec{z}}_\sU^k\Vert_\sF^{\lambda,k},\\
\ena
where $p_1, p_2, p_3$ are constants, and $\Vert\check{\bm{Z}}\Vert_\sF^{\lambda,k}$
, $\Vert \check{\bm{Y}}_\sV \Vert_\sF^{\lambda,k} $
, $\Vert\check{\vec z}_\sU\Vert_\sF ^{\lambda,k}$
, $\Vert {\bm{Y}} \Vert_\sF^{\lambda,k}$ are $\lambda$-sequences of 
$\Vert\check{\bm{Z}}^k\Vert_\sF$, $\Vert \check{\bm{Y}}_\sV^k \Vert_\sF$
, $\Vert\check{\vec z}_\sU^k\Vert_\sF$, $\Vert {\bm{Y}}^k\Vert_\sF$, respectively. 

Let 
$\vec d^{\lambda,k}=[\Vert\check{\bm{Z}}\Vert_\sF^{\lambda,k}; 
\Vert \check{\bm{Y}}_\sV \Vert_\sF^{\lambda,k}; 
\Vert\check{\vec z}_\sU\Vert_\sF ^{\lambda,k}; 
\Vert {\bm{Y}} \Vert_\sF^{\lambda,k}
]
$, then
$\vec d^{\lambda,k}\preceq \bm{Q}\vec d^{\lambda,k}+\vec p$, and
$$\bm{Q}=\begin{bmatrix}
	0 & 0 & 0 & \frac{2\eta n^2\sqrt{2\tilde{n}}\tilde{t}}{\kappa^2(n-1)}\\
	\frac{8\beta\sqrt{\tilde n}{\tilde t}bK}{\kappa^2(1-\kappa)} & 0 & 0 & \frac{4\eta\beta\sqrt{\tilde n}{\tilde t}bK}{\kappa^2(1-\kappa)}\\
	\frac{6bn\beta+8b^2K^2\beta\sqrt{n}}{(1-\eta\alpha\kappa n)\alpha\kappa} & \frac{2b}{(1-\eta\alpha\kappa n)\alpha\kappa} & 0 & \frac{2\eta b^2K^2\beta}{(1-\eta\alpha\kappa n)\alpha\kappa}\\
	2\beta\sqrt{n\tilde{n}}bK &  n & \beta\sqrt{n\tilde{n}}bK & 0\\
\end{bmatrix}.$$ 
Clearly, $\Vert Q\Vert<1$ if $\eta$ is sufficiently small. 
Then, $\vec d^{\lambda,k}$ is bounded and $\Vert\check{\bm{Z}}^{k}\Vert_\sF, \Vert\check{\vec z}_\sU^{k} \Vert_\sF$ converges to $0$ at the rate of $\cO(\lambda)$. 
Moreover, 
$$\Vert{\vec z}_i^{k}-\vz^* \Vert_\sF\le 2 \Vert\check{\bm{Z}}^{k}\Vert_\sF+\Vert\check{\vec z}_\sU^{k} \Vert_\sF, \forall i\in \cV, \forall k,$$
which leads to the result in the Theorem \ref{main_theo}.

An upper bound of $\eta$ can be given by bounding $\Vert \bm{Q}^4\Vert_\infty$ to ensure that the spectral radius of ${\bm Q}$ is strictly less than $1$. The bound is as described in
Theorem \ref{main_theo}, i.e., 
$\eta \in \left(0,\frac{\alpha \kappa^4(1-\kappa)^2}{72\beta^3 n^3 b^6 K^3\tilde{t}^2}\right)$. 
Noticing \eqref{bound-psi} from Lemma \ref{lemma0-1}, we further obtain the bound for $\eta_2=\eta\zeta$ in Remark \ref{rem-zeta}.

\subsection{Proofs of \eqref{ineq1}-\eqref{ineq4}}\label{Appendix-ineq}
\subsubsection{Proof of \eqref{ineq1}}\label{important1}
	Recalling the definition of $\check{\bm{Z}}^{k}$ in Appendix \ref{brief-proof}, it follows from \eqref{syn-update-a}, 
	we have
	\bea\label{claim1-1}
	&\Vert \check{\bm{Z}}^{k+\tilde{t}} \Vert_\sF = \Vert  \bm{T}_\sU^{k+\tilde{t}} {\bm{Z}}^{k+\tilde{t}} \Vert_\sF\le \Vert   \bm{T}_\sU^{k+\tilde{t}} \bm{\phi}_{\sR}^{k:k+\tilde t}  {\bm{Z}}^{k} \Vert_\sF
	+ \eta \sum_{t=0}^{\tilde t-1} \Vert   \bm{T}_\sU^{k+\tilde{t}} \bm{\phi}_{\sR}^{k:k+\tilde t} \bm{I}_a^{k+t}  {\bm{Y}}^{k+t} \Vert_\sF.
	\ena
	\par For the first term in \eqref{claim1-1}, we use the property that $\vu^k$ is row-stochastic and obtain 
	\bea\label{claim1-2}
	&\bm{T}_\sU^{k+\tilde{t}} \bm{\phi}_{\sR}^{k:k+\tilde t}=(\bm{I}_{\tilde{n}}-\bone_{\tilde n}(\vu^{k+\tilde{t}-1})^\T) \bm{\phi}_{\sR}^{k:k+\tilde t}\\
	&=(\bm{I}_{\tilde{n}}-\bone_{\tilde n}(\vu^{k+\tilde{t}-1})^\T) (\bm{\phi}_{\sR}^{k:k+\tilde t} - \bone_{\tilde n}(\vu^{k-1})^\T)- (\bm{I}_{\tilde{n}}-\bone_{\tilde n}(\vu^{k+\tilde{t}-1})^\T) (\bone_{\tilde n}(\bm{\phi}_\sR^{k:k+\tilde t})^\T - \bone_{\tilde n}(\vu^{k-1})^\T)
	\\&=(\bm{I}_{\tilde{n}}-\bone_{\tilde n}(\vu^{k+\tilde{t}-1})^\T) (\bm{\phi}_{\sR}^{k:k+\tilde t}-\bone_{\tilde n}(\bm{\phi}_\sR^{k:k+\tilde t})^\T)(\bm{I}_{\tilde{n}}-\bone_{\tilde n}(\vu^{k-1})^\T)
	\\&=\bm{T}_\sU^{k+\tilde{t}}(\bm{\phi}_{\sR}^{k:k+\tilde t}-\bone_{\tilde n}(\bm{\phi}_\sR^{k:k+\tilde t})^\T)\bm{T}_\sU^k,
	\ena
	Since $\Vert\bm{A}\bm{B}\Vert_\sF\le \Vert\bm{A}\Vert_2\Vert\bm{B}\Vert_\sF$ and $\Vert\bm{A}\bm{B}\Vert_\sF\le \Vert\bm{A}\Vert_\sF \Vert\bm{B}\Vert_\sF$, we obtain
	\bea\label{claim1-3}
	&\Vert  \bm{T}_\sU^{k+\tilde t} \bm{\phi}_{\sR}^{k:k+\tilde t}  {\bm{Z}}^k \Vert_\sF
	\le \Vert  \bm{T}_\sU^{k+\tilde t}\Vert_2
	\Vert \bm{\phi}_{\sR}^{k:k+\tilde t}-\bone_{\tilde n}\bm{\phi}_\sR^{k:k+\tilde t} \Vert_\sF \Vert \bm{T}_\sU^k {\bm{Z}}^k \Vert_\sF
	\le 2\mu \Vert \check{\bm{Z}}^k \Vert_\sF,
	\ena
	where the last inequality follows from $\Vert  \bm{T}_\sU^{k+\tilde t}\Vert_2=\Vert  \bm{I}_{\tilde{n}}-\bone_{\tilde n}(\vu^{k+\tilde{t}-1})^\T\Vert_2<2$ and Corollary \ref{coro1}.
	\par For the second term in \eqref{claim1-1}, it holds that
	\bea\label{claim1-4}
	&\eta \sum_{t=0}^{\tilde t-1} \Vert  \bm{T}_\sU^{k+\tilde t} \bm{\phi}_{\sR}^{k:k+t} \bm{I}_a^{k+t}  {\bm{Y}}^{k+t} \Vert_\sF\le\eta \Vert  \bm{T}_\sU^{k+\tilde t}\Vert_2\sum_{t=0}^{\tilde t-1} \Vert\bm{\phi}_{\sR}^{k:k+t} \bm{I}_a^{k+t}  {\bm{Y}}^{k+t} \Vert_\sF\le \eta \sqrt{2\tilde n} \sum_{t=0}^{\tilde t-1}\Vert {\bm{Y}}^{k+t} \Vert_\sF,
	\ena
	where $\Vert \bm{\phi}_{\sR}^{k:k+t}\Vert_2\le \sqrt{\tilde n/2},\forall t\ge 0$. 
	The desired result then follows by combining \eqref{claim1-3} and \eqref{claim1-4}.

\subsubsection{Proof of \eqref{ineq2}}\label{important2}

\par To study $ {\bm{Y}}_\sV^k$ defined in \eqref{S_V}, let
\bea\label{claim2-last}
 {\bm{H}}_\sV^k=(\bm{V}^{k+1})^\dag {\bm{H}}_{\sC}^k\bm{V}^k.\\
\ena
Furthermore, we define 
$
\bm{\Phi}_{\sV}^{k:k+t}=\prod_{j=k}^{k+t-1} {\bm{H}}_\sV^j,
$
and demonstrate that $\bm{\Phi}_{\sV}^{k:k+t}$ converges to a rank-one matrix, i.e.,
\bee\label{time-large}
\left\Vert \bm{\Phi}_{\sV}^{k:k+\tilde t} - \frac1n \bone_\sV^k (\vec v^k)^\T \right\Vert_\sF < \kappa^{-1}\mu n<\frac12.
\ene
To show \eqref{time-large}, first,
$$\begin{aligned}
\bm{\Phi}_{\sV}^{k:k+t}&= {\bm{H}}_\sV^{k+t-1}\ldots  {\bm{H}}_\sV^k\\
&= (\bm{V}^{k+t})^\dag  {\bm{H}}_{\sC}^{k+t-1}\bm{I}_\sV^{k+t-1}\ldots  {\bm{H}}_{\sC}^k\bm{V}^k\\
&= (\bm{V}^{k+t})^\dag \bm{\Phi}_{\sC}^{k:k+t}\bm{V}^k.
\end{aligned}$$
The last equality is tenable because $\bm{I}_\sV^{k+1} {\bm{H}}_{\sC}^k\bm{V}^k= {\bm{H}}_{\sC}^k\bm{V}^k, \forall k\ge 0$, which together with $\displaystyle \vec v^{k+1}= {\bm{H}}_{\sC}^k\vec v^k$, can be verified by
computing each row on both sides.
\par By Lemma \ref{lemma00}(a), there exists $\Delta\bm{\Phi}_{\sC}^{k:k+t}\in \mathbb{R}^{\tilde n\times \tilde n}$ such that $\Vert \Delta\bm{\Phi}_{\sC}^{k:k+t} \Vert_\sF\le 2\delta^t$, and 
$
\bm{\phi}_{\sC}^{k:k+t}\bone_{\tilde n}^\T+\Delta\bm{\Phi}_{\sC}^{k:k+t}=\bm{\Phi}_{\sC}^{k:k+t}.
$
Right multiplying $\vec v^k$ on both sides yields
$
\bm{\phi}_{\sC}^{k:k+t}\bone_{\tilde n}^\T\vec v^k+\Delta\bm{\Phi}_{\sC}^{k:k+t} \vec v^k
=\bm{\Phi}_{\sC}^{k:k+t} \vec v^k=\vec v^{k+t}.
$
Since $\bm{\phi}_{\sC}^{k:k+t}\bone_{\tilde n}^\T \vec v^k=\bm{\phi}_{\sC}^{k:k+t}\bone_{\tilde n}^\T \vec v^0=n\bm{\phi}_{\sC}^{k:k+t}$, we have
$
\bm{\phi}_{\sC}^{k:k+t}=\frac1n (\vec v^{k+t} - \Delta\bm{\Phi}_{\sC}^{k:k+t} \vec v^k).
$
Then,
\bea\label{result1}
\left\Vert \bm{\Phi}_{\sV}^{k:k+t} - \frac{\bone_\sV^k(\vec v^k)^\T}{n}  \right\Vert_\sF&=\left\Vert \bm{\Phi}_{\sV}^{k:k+t} - \frac{(\bm{V}^{k+t})^\dag \vec v^{k+t}\bone_{\tilde n}^\T \bm{V}^k}{n} \right\Vert_\sF\\
&=\left\Vert (\bm{V}^{k+t})^\dag \Delta\bm{\Phi}_{\sC}^{k:k+t}\left(\bm{I}_{\tilde n} - \frac{\vec v^k \bone_{\tilde n}^\T}{n} \right)\right\Vert_\sF\\
&\le \Vert (\bm{V}^{k+t})^\dag\Vert_2 \Vert\Delta\bm{\Phi}_{\sC}^{k:k+t}\Vert_\sF \left\Vert\left(\bm{I}_{\tilde n} - \frac{\vec v^k \bone_{\tilde n}^\T}{n} \right)\right\Vert_\sF\\
&< \kappa^{-1}\cdot 2\delta^t \cdot n,
\ena
where we have used the fact that all the entries of $(\bm{V}^k)^\dag$ are less than $\kappa^{-1}$ from Lemma \ref{lemma00}(b).
\eqref{time-large} follows directly by substituting $\tilde t$ into \eqref{result1}.

\par Define $\bm{D}^k=\partial J^{k+1}- \partial J^k$ and rewrite the update \eqref{syn-update-b} as
\bee\label{syn-b-rewrite}
 {\bm{Y}}_\sV^{k+1} =  {\bm{H}}_\sV^k  {\bm{Y}}_\sV^k + (\bm{V}^{k+1})^\dag \bm{D}^k.
\ene
Then, $\beta$ is from Lemma \ref{lemma0-1}, and
$\bm{D}^k$ is bounded for all $k\ge 0$ as below
	\bea%\label{time-long}
	\label{eq-lemma10-1}
	\Vert\bm{D}^k\Vert_\sF
	&= \Vert \nabla j_{i,p_{k+1}}(\vec z^{k+1}_{i})- \nabla j_{i,p_{k+1}}(\vec z^{\tau_{i,p^{k+1}}^{k}}_{i}) \Vert_\sF\\
	&\le \beta \sum_{t=\tau_{i,p^k}^{k}}^{k} \Vert\vec z^{t+1}_{i}-\vec z^{t}_{i}\Vert\\
	&\le \beta \sum_{t=k-bK}^{k} \Vert\vec z^{t+1}_{i}-\vec z^{t}_{i}\Vert\\
	&\le \beta \sum_{t=k-bK}^{k} \Vert\bm{I}_a^k( {\bm{Z}}^{t+1}- {\bm{Z}}^t)\Vert_\sF\\
	&\le \beta\sum_{t=k-bK}^{k}\left(\Vert(\bm{I}_a^k {\bm{H}}_\sR^t-\bm{I}_a^k)\bm{T}_\sU^t {\bm{Z}}^t)\Vert_\sF + \eta \Vert {\bm{Y}}^t\Vert_\sF \right)\\
	&\le \beta\sum_{t=k-bK}^{k}\left(2\Vert\check{\bm{Z}}^t\Vert_\sF + \eta \Vert {\bm{Y}}^t\Vert_\sF \right)
	\ena
	where the second inequality follows Assumption 1(b), 2(b), and
	the last from the row-stochasticity of $ {\bm{H}}_\sR^k$.

Next, we prove \eqref{ineq2} in a similar way as in Section \ref{important1}.
By \eqref{syn-b-rewrite}, we get
\bea\label{claim2-20}
\Vert \check{\bm{Y}}_\sV^{k+\tilde{t}} \Vert_\sF &= \left\Vert \bm{T}_\sC^{k+\tilde{t}}  {\bm{Y}}_{\sV}^{k+\tilde{t}} \right\Vert_\sF\\
&\le \left\Vert  \bm{T}_\sC^{k+\tilde{t}} \bm{\Phi}_{\sV}^{k:k+\tilde t} {\bm{Y}}_{\sV}^k  \right\Vert_\sF
+ \sum_{t=0}^{\tilde t-1} \left\Vert \left(\bm{I}_{\sV}^{k+\tilde{t}}-\frac1n \bone_\sV^{k+\tilde{t}} (\vec v^{k+\tilde{t}})^\T \right)\right.\left.\bm{\Phi}_\sV^{k+t+1:k+\tilde t}(\bm{V}^{k+t+1})^\dag \bm{D}^{k+t}  \right\Vert_\sF\\
&\le 2\kappa^{-1}\mu n \Vert \check{\bm{Y}}_\sV^k \Vert_\sF + \sum_{t=0}^{\tilde t - 1} 2\sqrt{\tilde n}\kappa^{-1} \Vert \bm{D}^{k+t} \Vert_\sF,
\ena
where the last inequality is from \eqref{time-large}, and the fact that $\bm{T}_\sC^{k+\tilde{t}} \bm{\Phi}_\sV^{k+\tilde{t}}=\bm{T}_\sC^{k+\tilde{t}} \left[\bm{\Phi}_{\sV}^{k:k+\tilde t} - \frac1n \bone_\sV^{k+\tilde{t}} (\vec v^{k})^\T\right]\bm{T}_\sC^k$ 
can be verified in the same way as in \eqref{claim1-2}. 
\eqref{claim2-20} along with \eqref{eq-lemma10-1} implies the desired result \eqref{ineq2}. 

\subsubsection{Proof of \eqref{ineq3}}\label{important3}
Define a gradient matrix $\nabla^k$ in the same form as $\partial^k$, i.e.,
	\bea\label{eq-lemma10-0}
	\nabla^k&=[\nabla J^k; \bzero_{bn\times 2d}]\in\bR^{\tilde n\times 2d},\\
	\nabla J^k&=[\nabla J_1^k; \ldots; \nabla J_n^k]\in\bR^{n\times 2d},\\
	\nabla J_i^k &= \nabla j_i(\vec z^k_{i}), \ 1\le i \le n.
	\ena
By definition,
$(\check{\vec z}_\sU^k)^\T 
=(\vu^k)^\T {\bm{Z}}^k - (\vec z^*)^\T$. Repeating the first step in the above two proofs,
we get 
\bea\label{claim3-1}
(\vec z_\sU^{k+b})^\T&=(\vu^{k+b})^\T {\bm{Z}}^{k+b}\\
&=(\vu^{k+b})^\T \bm{\Phi}_\sR^{k:k+b}  {\bm{Z}}^{k+b}- \eta (\vu^{k+b})^\T 
\sum_{t=0}^{b-1} \bm{\Phi}_\sR^{k+t+1:k+b}\bm{I}_a^{k+t}\bm{V}^{k+t} {\bm{Y}}_\sV^{k+t}\\
&=(\vec z_\sU^k)^\T - \eta (\vu^{k+b})^\T 
\sum_{t=0}^{b-1} \bm{\Phi}_\sR^{k+t+1:k+b}\bm{I}_a^{k+t}\bm{V}^{k+t} {\bm{Y}}_\sV^{k+t}.\\
\ena
We  extract the gradient of $J(\cdot)$ from \eqref{claim3-1} by defining
$$\begin{aligned}
(\vec r^{k,t})^\T &= \eta (\vu^{k+b})^\T\bm{\Phi}_\sR^{k+t+1:k+b}\bm{I}_a^{k+t}\bm{V}^{k+t},\\
\eta^{k,t} &= (\vec r^{k,t})^\T\bone_\sV^{k+t},
\ \underline{\eta}^k = \sum_{t=0}^{b-1}\eta^{k,t}.
\end{aligned}$$
By $\nabla j(\vec z_\sU^k)$, the global gradient of $\vec z_\sU^k$, the last row in \eqref{claim3-1} is decomposed into
\bea\label{claim3-3}
& (\vec z_\sU^{k+b} - \vec z^*)^\T
=(\vec z_\sU^{k} - \vec z^* - \underline{\eta}^k\nabla j(\vec z_\sU^k))^\T\\
&\quad+\sum_{t=0}^{ b-1}\eta^{k,t}\left(\nabla j(\vec z_\sU^{k})^\T-\bone_{\tilde n}^\T\nabla^{k+t}\right)\\
&\quad-\sum_{t=0}^{b-1}(\vec r^{k,t})^\T\left( {\bm{Y}}_\sV^{k+t}-\bone_\sV^{k+t} \bone_{\tilde n}^\T \partial^{k+t}\right)\\
&\quad+ \sum_{t=0}^{b-1}\eta^{k,t}\left[\bone_{\tilde n}^\T (\nabla^{k+t}-\partial^{k+t})\right].\\
\ena
The first row of \eqref{claim3-3}, $\vec z_\sU^{k} - \vec z^* - \underline{\eta}^k\nabla j(\vec z_\sU^k)$, depicts the gradient ascent/descent of the weight average variables
$\vec z_\sU^k$, while the second to the fourth rows are bounded error terms, as proved in \eqref{claim3-5}-\eqref{claim3-7} below.
\par First, we have $\eta^{k,t}\le \eta n$ and $\underline{\eta}^k\le \eta bn$. Applying Lemma \ref{lemma0-1}, we have
$
\Vert \vec z_\sU^{k} - \vec z^* - \underline{\eta}^k\nabla j(\vec z_\sU^k)\Vert_\sF
\le (1- \eta\alpha bn)\left\Vert \vec z^k_\sU- \vec z^* \right\Vert_\sF.
$

\par Second, let $\bone_{n|\tilde n}=[\bone_n; \bzero_{\tilde n-n}]\in\mathbb{R}^{\tilde n}$, and $\bm{I}_{n|\tilde n}=\text{diag}(\bone_{n|\tilde n})\in \mathbb{R}^{\tilde n\times\tilde n}$.
By direct computation, $\nabla j(\vec z)^\T=\bone_{n}^\T\nabla J(\bone_n\vec z^\T),\,\forall \vec z$. From \eqref{eq-lemma10-0}, we obtain that for all $k$, $\bone_{\tilde n}^\T\nabla^k=\bone_{n|\tilde n}^\T\nabla^k=\bone_{n}^\T \nabla J\left(\bm{Z}_0^{k}\right)$.
Then,
\bea\label{claim3-5}
&\left\Vert\sum_{t=0}^{b-1}\eta^{k,t}\left(\nabla j(\vec z_\sU^k)^\T - \bone_{\tilde n}^\T\nabla^{k+t}\right)\right\Vert_\sF
\\&=\left\Vert\sum_{t=0}^{b-1}\frac{\eta^{k,t}}{n}\left(\bone_{n}^\T\nabla J(\bone_n(\vec z_\sU^k)^\T)
-\bone_{n}^\T \nabla J(\bm{Z}_0^{k+t})\right)\right\Vert_\sF\\
&\le  \eta\sqrt{n}\beta \sum_{t=0}^{b-1}\left\Vert\bone_{n|\tilde n}^\T(\vec z_\sU^k)^\T 
-\bone_{n|\tilde n}^\T\vu^{k+t-1} {\bm{Z}}^{k+t}\right\Vert_\sF  
\\&\le 3\eta n\beta\sum_{t=0}^{b-1} \Vert \check{\bm{Z}}^{k+t}\Vert_\sF,
\ena
where the first inequality follows from \eqref{lem-beta}.
Third,  it follows from \eqref{lemma-tracking} and \eqref{S_V} that
\bea\label{claim3-6}
&\left\Vert\sum_{t=0}^{b-1}(\vec r^{k,t})^\T\left( {\bm{Y}}_\sV^{k+t}-\bone_\sV^{k+t} \bone_{\tilde n}^\T \partial^{k+t}\right)\right\Vert_\sF
\\&=\left\Vert\sum_{t=0}^{b-1}(\vec r^{k,t})^\T\left( {\bm{Y}}_\sV^{k+t}-\bone_\sV^{k+t}(\vv^{k+t})^\T  {\bm{Y}}_\sV^{k+t}\right)\right\Vert_\sF
\\&\le \sum_{t=0}^{b-1}\Vert \vec r^{k,t} \Vert_2 \Vert \check{\bm{Y}}_{\sV}^{k+t} \Vert_\sF\le \eta n\sum_{t=0}^{b-1}\Vert\check{\bm{Y}}_{\sV}^{k+t} \Vert_\sF.
\ena
Moreover, similar to \eqref{eq-lemma10-1}, we obtain,
\bea\label{claim3-7}
&\left\Vert\sum_{t=0}^{b-1}\eta^{k,t}\left[\bone_{\tilde n}^\T (\nabla^{k+t}-\partial^{k+t})\right]\right\Vert_\sF\\
&\le \eta n \sum_{t=0}^{b-1} \left\Vert\sum_{i=1}^n\sum_{p=1}^{m_i}\nabla j_{i,p}(\vec z_i^k) - \nabla j_{i,p}(\vec z_i^{\tau_{i,p}^{k}})  \right\Vert\\
&\le \eta n \beta \sum_{t=0}^{b-1}\sum_{j=t-bK+1}^{t}\Vert\bm{I}_{n|\tilde n}( {\bm{Z}}^{t+1}- {\bm{Z}}^t)\Vert_\sF \\
&\le \eta n \beta \sum_{t=0}^{b-1}\sum_{j=t-bK+1}^{t}\left(2\sqrt{n}\Vert\check{\bm{Z}}^j\Vert_\sF + \eta \Vert {\bm{Y}}^j\Vert_\sF \right).
\ena
Then, one can summarize \eqref{claim3-3}-\eqref{claim3-7} to get \eqref{ineq3}.
\subsubsection{Proof of \eqref{ineq4}}\label{important4}
By \eqref{S_V} and \eqref{lemma-tracking}, it holds that $\bone_{\tilde n}^\T \partial^k=(\vv^{k+t})^\T  {\bm{Y}}_{\sV}^k$ and
\bea\label{claim4-1}
&\Vert {\bm{Y}}^k\Vert_\sF
=\Vert \bm{V}^k\bm{T}_\sC^k {\bm{Y}}_{\sV}^k+\frac{1}{n}\bm{V}^k(\bone_\sV^k)^\T(\vv^{k+t})^\T  {\bm{Y}}_{\sV}^k\Vert_\sF\\
&\le \Vert \bm{V}^k\Vert_2 \Vert \bm{T}_\sC^k  {\bm{Y}}_{\sV}^k\Vert_\sF + \left\Vert \frac{1}{n}\bm{V}^k\bone_\sV^k\bone_{\tilde n}^\T\partial^k \right\Vert_\sF
\\&\le n \Vert \check{\bm{Y}}_{\sV}^k \Vert_\sF + \Vert \bone_\sV^k\bone_{\tilde{n}}^\T\partial^k \Vert_\sF.
\ena 
For the second term in \eqref{claim4-1}, we have
\bea\label{claim4-2}
&\Vert\bone_\sV^k\bone_{\tilde n}^\T \partial^k\Vert_\sF=\left\Vert\bone_\sV^k\left(\bone_{n}^\T \partial J^k-\bone_n^\T\nabla J(\bone_n(\vec z^\star)^\T)\right)\right\Vert_\sF\\
&\le \sqrt{\tilde{n}}   \left\Vert \sum_{p=1}^{m_i}\sum_{i=1}^n (j_{i,p}(\vec z_i^{\tau_{i,p}^{k}})-j_{i,p}(\vec z^*))\right\Vert\\
&\le \beta\sqrt{\tilde{n}} \sum_{t=k-bK}^{k-1}\Vert \bm{I}_{n|\tilde n} {\bm{Z}}^t - \bone_{n|\tilde n}(\vec z^*)^\T\Vert_\sF\\
&\le \beta\sqrt{\tilde{n}} \sum_{t=k-bK}^{k}(2\sqrt{n}\left\Vert \check{\bm{Z}}^t\right\Vert_\sF + \sqrt{n}\left\Vert \vec z_\sU^t-\vec z^*\right\Vert_\sF) \\
\ena
Thus, we obtain \eqref{ineq4} by using \eqref{claim4-1} and \eqref{claim4-2}.

\bibliography{vref}
\bibliographystyle{plain}

\end{document}